\def\year{2022}\relax
\documentclass[letterpaper]{article} 
\usepackage{aaai22}  
\usepackage{times}  
\usepackage{helvet}  
\usepackage{courier}  
\usepackage[hyphens]{url}  
\usepackage{graphicx} 
\usepackage{natbib}  
\usepackage{caption} 
\DeclareCaptionStyle{ruled}{labelfont=normalfont,labelsep=colon,strut=off} 
\usepackage{subfigure}
\usepackage{amsfonts,amsthm,amssymb}
\usepackage{amsmath}

\newtheorem{lemma}{Lemma}
\newtheorem{theorem}[lemma]{Theorem}

\newcommand{\predT}{\hat{T}}

\usepackage[ruled]{algorithm2e}

\newcommand{\cI}{{\cal I}}

\newcommand{\opt}{\textsc{OPT}{}}

\newcommand{\mst}{\mathrm{MST}}
\newcommand{\mdst}{\mathrm{MDST}}

\newcommand{\newdelta}{\mathop{\Delta}\limits}

\usepackage{color}

\title{Learning-Augmented Algorithms for Online Steiner Tree}
\author{
   Chenyang Xu\textsuperscript{\rm 1} \thanks{The corresponding author.} and
    Benjamin Moseley\textsuperscript{\rm 2} 
}

\affiliations{
    \textsuperscript{\rm 1}College of Computer Science, Zhejiang University, $xcy1995@zju.edu.cn$

\textsuperscript{\rm 2}Tepper School of Business, Carnegie Mellon University, $moseleyb@andrew.cmu.edu $ 
}

\begin{document}

    \maketitle

\begin{abstract}
This paper considers the recently popular beyond-worst-case algorithm analysis model which integrates machine-learned predictions with online algorithm design.  We consider the online Steiner tree problem in this model for both directed and undirected graphs.  Steiner tree is known to have strong  lower bounds in the online setting and any algorithm's worst-case guarantee is far from desirable.  

This paper considers algorithms that predict which terminal arrives online.  The predictions may be incorrect and the algorithms' performance is parameterized by the number of incorrectly predicted terminals. These guarantees ensure that algorithms break through the online lower bounds with good predictions and the competitive ratio gracefully degrades as the prediction error grows.  We then observe that the theory is predictive of what will occur empirically.  We show on graphs where terminals are drawn from a distribution, the new online algorithms have strong performance even with modestly correct predictions. 
\end{abstract}

\section{Introduction}

An emerging line of work on beyond-worst-case algorithms makes use of machine learning for algorithmic design.  This line of work suggests that there is an opportunity to advance the area of beyond-worst-case algorithmics and analysis by 
augmenting
combinatorial algorithms with machine learned predictions. Such  algorithms perform better than worst-case bounds with accurate predictions while retaining the worst-case guarantees even with erroneous predictions. There has been significant interest in this area (e.g.~\cite{DBLP:journals/siamcomp/GuptaR17,DBLP:conf/icml/BalcanDSV18,DBLP:conf/nips/BalcanDW18,ChawlaGTTZ19,Kraska,MLCachingLykouris,PurohitNIPS,DBLP:conf/soda/LattanziLMV20}).


\paragraph{Online Learning-Augmented Algorithms.} This paper considers the augmenting model in the online setting where algorithms make decisions over time without knowledge of the future.  In this model, an algorithm is given access to a learned prediction about the problem instance. 
The learned prediction is error prone and the performance of the algorithm is expected to be bounded in terms of the prediction's quality. The quality measure is prediction specific. 
The performance measure is the \emph{competitive ratio} where an algorithm is $c$-competitive if the algorithm's objective value is at most a $c$ factor larger than the optimal objective value \emph{for every input}. 
In the learning-augmented algorithms model, finding appropriate parameters to predict and making the algorithm robust to the prediction error are usually key algorithmic challenges.


Many online problems have been considered in this context, such as caching~\cite{MLCachingLykouris,DBLP:conf/soda/Rohatgi20,DBLP:conf/icalp/JiangP020,WeiCaching20}, page migration~\cite{IMMR20}, metrical task systems~\cite{ACEPS20}, ski rental~\cite{PurohitNIPS,DBLP:conf/icml/GollapudiP19,anand2020customizing}, scheduling~\cite{PurohitNIPS,DBLP:conf/spaa/Im0QP21}, load balancing~\cite{DBLP:conf/soda/LattanziLMV20}, online linear optimization~\cite{BhaskarOnlineLearning20},
online flow allocation~\cite{DBLP:conf/esa/LavastidaM0X21},
speed scaling~\cite{DBLP:conf/nips/BamasMRS20}, set cover~\cite{DBLP:conf/nips/BamasMS20}, and bipartite matching and secretary problems~\cite{DBLP:conf/nips/AntoniadisGKK20}.

\paragraph{The Steiner Tree Problem.} Steiner tree is one of the most fundamental combinatorial optimization problems. For undirected Steiner tree, there is an undirected graph $G= (V,E)$ where each edge $e \in E$ has a cost $c_e$ and a terminal set $T \subseteq V$. We need to buy edges in $E$ such that all terminals are connected via the bought edges and the goal is to minimize the total cost of the bought edges.
For the directed case, the edges are directed and there is a root node $r$.  In this problem all of the terminals must have a directed path to the root via the edges bought.  

Theoretically, the problem has been of interest to the community for decades, starting with the inclusion in Karp's 21 NP-Complete problems~\cite{DBLP:conf/coco/Karp72}. Since then, it has been studied extensively in approximation algorithm design~\cite{DBLP:journals/acta/KouMB81,takahashi1980approximate,wu1986faster,DBLP:conf/stoc/ByrkaGRS10}, stochastic algorithms~\cite{DBLP:conf/icalp/GuptaP05,DBLP:conf/approx/GuptaHK07,DBLP:conf/memics/KurzMZ12,DBLP:journals/coap/LeitnerLLS18} and online algorithms \cite{DBLP:journals/siamdm/ImaseW91,f3c9525d5573414ab733cd3433bdca06,DBLP:conf/esa/Angelopoulos08,DBLP:conf/dagstuhl/Angelopoulos09}. Practically, the Steiner tree problem is fundamental for many network problems such as fiber optic networks~\cite{DBLP:conf/or/BachhieslPPWS02}, social networks~\cite{DBLP:conf/aistats/ChiangLLP13,DBLP:conf/kdd/LappasTGM10}, and biological networks~\cite{DBLP:journals/bmcbi/SadeghiF13}. This problem has so many uses practically, that recently there have been competitions to find fast algorithms for it and its variants, including the 11th DIMACS Implementation Challenge (2014) and the 3rd Parameterized Algorithms and Computational Experiments (PACE) Challenge (2018).

This paper focuses on the online version of Steiner tree. 
In this case, the graph $G$ is known in advance, meaning that the edges that can be bought are completely known as well as all the nodes in the graph.  However, the nodes that actually are the terminals $T$ are unknown. The terminals in $T$ arrive one at a time. Let $t_1 , t_2, \ldots t_k$ be the arriving order of terminals, where $k=|T|$. When terminal $t_i$ arrives, it must immediately be connected to $t_1, t_2, \ldots t_{i-1}$ by buying edges of $G$ and once an edge is bought, it is irrevocable. The goal is to minimize the total cost.

The online problem occurs often in practice. For instance, when building a network often new nodes are added to the network over time. Not knowing which terminals will arrive makes the problem inherently hard. The algorithm with the best worst-case guarantees is the simple greedy algorithm~\cite{DBLP:journals/siamdm/ImaseW91}, which always chooses to connect an arriving node via the cheapest feasible path. The competitive ratios of the greedy algorithm on undirected graphs and directed graphs are, respectively,  $\Theta(\log k)$ and $\Theta(k)$, which are the best possible using worst-case analysis (see~\cite{DBLP:journals/siamdm/ImaseW91,DBLP:journals/ipl/WestbrookY95}). 
However, these results are far from desirable. The question thus looms, is there the potential to go beyond worst-case lower bounds in the learning-augmented algorithms for online Steiner tree?

\subsection{Results}
We consider the online Steiner tree problem in the learning-augmented model. 
The prediction is defined to be the set of terminals. That is, the algorithm is supplied with a set of terminals $\predT$ at the beginning of time. Some of these may be incorrect. Define the prediction error $\eta$ to be the number of incorrectly predicted terminals. Then the actual terminals in $T$ arrive online. This paper shows the following results, breaking through worst-case lower bounds.

\begin{itemize}
\item  
In the undirected case, we propose an $O(\log \eta)$-competitive algorithm.  That is, with accurate predictions, the algorithm is \emph{constant competitive}.  Then with the worst predictions, the competitive ratio is $O(\log k)$, \emph{matching the best worst-case bound}.  Between, the algorithm has slow degradation of performance in terms of the prediction error. We further show that any algorithm has competitive ratio $\Omega( \log \eta)$ with this prediction and thus our algorithm is the best possible online algorithm using this prediction.
\item In the directed case, we give an algorithm that is $O( \eta + \log k)$-competitive.   With near perfect predictions, the algorithm is $O(\log k)$-competitive, which is exponentially better than the worst-case lower bound $\Omega(k)$.  With a large prediction error, the algorithm matches the $O(k)$ bound of the best worst-case algorithm.  Between, the algorithm has slow degradation of performance in terms of the error as in the undirected case. As in the undirected case, we show that any algorithm has competitive ratio $\Omega(\eta)$ with this prediction. Our algorithm is close to the best possible when using this prediction.
\end{itemize}

The next question is if these theoretical results predict what will occur empirically on real graphs.  For the undirected case we show that with modestly accurate predictions, the algorithms indeed can outperform the baseline. Then the performance degrades as there is more error in the prediction, never becoming much worse than the baseline. These empirical results corroborate the theory. Moreover, we give a learning algorithm that is able to learn predictions from a small number of sample instances such that our Steiner tree algorithms have strong performance.

\section{Online Undirected Steiner Tree}\label{sec:undirected_case}
For the brevity of the algorithms' statement and analysis, we make two assumptions. First, we assume that $G$ is a complete graph in metric space.  This can be assumed by taking the metric completion of any input graph and is standard for the Steiner tree problem. Second, the predicted terminal set $\predT$ and the real terminal set $T$ share the same size $k$. In Appendix~\ref{sec:omit}, we show this assumption can be removed easily.
We aim to show the following theorem in this section.
	
\begin{theorem}\label{thm:main}
	Given a predicted terminal set $\predT$, there exists an algorithm with competitive ratio at most $O(\log \eta)$, where $\eta := k-|T\cap \predT|$.
\end{theorem}
	

\subsection{Preliminaries}
\label{sec:prelim}

The input is an undirected graph $G=(V,E)$, where each edge $e$ has cost $c_e\geq 0$, 
and a terminal set $T\subseteq V$ that arrives online. Recall $k:=|T| = |\hat{T}|$. When a terminal $t$ arrives, we must buy some edges such that it is connected with all previous terminals in the subgraph formed by bought edges. 
The goal is to minimize the total cost of the bought edges.

In the analysis, we will leverage results on the online greedy algorithm.  The following theorem was shown in \cite{DBLP:journals/siamdm/ImaseW91}. The traditional online greedy algorithm maintains a tree $T$ connecting all the terminals. This tree is initialized to $\emptyset$.  Then when a terminal $t$ arrives, 
the edges on the shortest path from $t$ to any node in $T$ will be added into $T$. 

\begin{theorem}[\cite{DBLP:journals/siamdm/ImaseW91}] \label{thm:greedy}
The online greedy algorithm is $O(\log k)$-competitive. 
\end{theorem}

 We will also use the following properties of minimum spanning trees.
 
 \begin{lemma}\label{lem:cycle}
Consider an offline Steiner tree instance. 
A minimum spanning tree $\mst(T)$ on terminals is a 2-approximated solution~\cite{DBLP:journals/acta/KouMB81}.
In addition, for any edge $e \notin \mst(T)$, the cost of $e$ is at least as large as the minimum cost of edges in the unique cycle in $\mst(T) \cup \{e\}$~\cite{schrijver2003combinatorial}.

 \end{lemma}



\subsection{Warm-up: Analysis of a Simple Online Algorithm}
Towards proving Theorem~\ref{thm:main}, we first introduce a simple and natural algorithm whose competitive ratio is $O(\eta)$.  This is a far worse guarantee than the algorithm we  develop, but it will help build our techniques and give the intuition. 

Intuitively, if the prediction is error-free, the instance becomes an offline problem. Several constant approximation algorithms can be employed for the offline case. For example, we compute a minimum spanning tree $\mst(\predT)$ on the accurate predicted terminal set $\predT$ and each time when a new terminal arrives, connect it with all previous terminals only using the edges in $\mst(\predT)$. This algorithm obtains a competitive ratio 2 if $\predT = T$.

Inspired by this, a natural online algorithm is the following. This algorithm has poor performance when the error in the predictions is large. This will then lead us to develop a more robust algorithm.

\medskip
\noindent \textbf{Online Algorithm with Predicted Terminals (OAPT)\footnote{The pseudo-code of all proposed algorithms in this paper are provided in Appendix~\ref{sec:code}.}:}  
Let $\predT$ be the predicted set of terminals and $\mst(\predT)$ be the minimum spanning tree on $\predT$.  Let $T_{i}$ be the first set of $i$ terminals that arrive online.  $T_k$ contains all online terminals.

Initialize $A = \emptyset$ to be the tree that the algorithm will construct connecting the online terminals. The algorithm returns the set of edges in $A$ after all terminals arrive.  
We divide the edges of $A= A_1 \cup A_2$ into two sets, $A_1$ and $A_2$ depending on the case that causes us to add edges to $A$.    Consider when terminal $t_i$ arrives.

\begin{itemize}
\item \textbf{Case 1:}  If $t_i \notin \predT$ or $t_i$ is the first terminal in $\predT$ to arrive, add to $A_1$ the shortest edge in $G$ connecting $t_i$ to terminals $T_{i-1}$ that have arrived. No edge is bought if this is the first terminal that arrives. 
\item \textbf{Case 2:} Otherwise, add the shortest path in $\mst(\predT)$ to $A_2$ which connects $t_i$ to a terminal in $\predT \cap T_{i-1}$.  In other words, buy the shortest path in $\mst(\predT)$ connecting $t_i$ to a \emph{predicted} terminal that has previously arrived.
\end{itemize}

 
 

Our goal is to show that the competitive ratio of this algorithm is exactly $\Theta(\eta)$.  

\begin{theorem}\label{thm:oapt}
The competitive ratio of  OAPT is $\Theta(\eta)$. 
\end{theorem}


First we observe that the algorithm is no better than $\Omega(\eta)$-competitive.  This lower bound will motivate the design of a more robust algorithm in the next section.  

\begin{lemma}\label{lem:lower}
The competitive ratio of  OAPT is $\Omega(\eta)$.
\end{lemma}

\begin{figure}[htbp]
    \centering
    \includegraphics[width=0.6\linewidth]{ 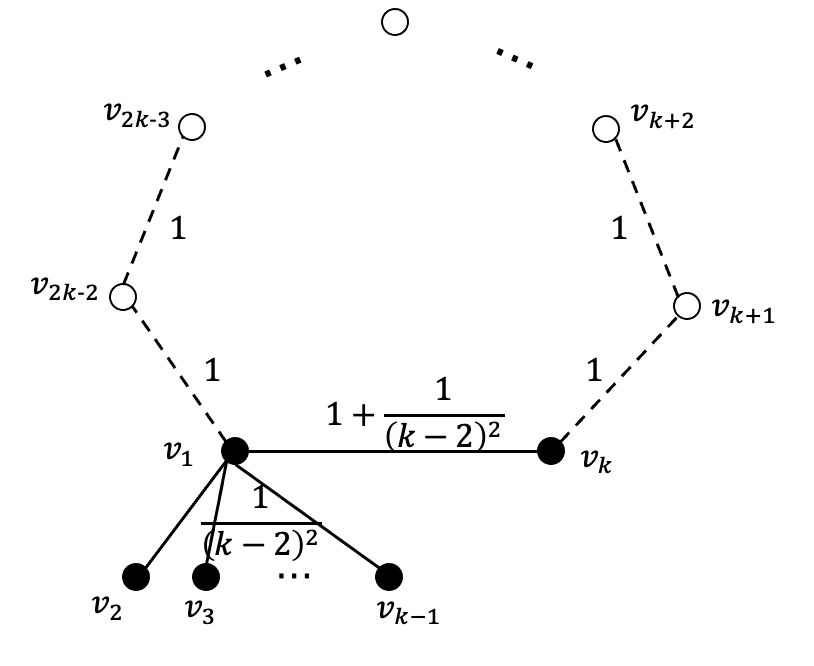}
    \caption{An online Steiner tree instance. The graph has $2k-2$ nodes and $2k-2$ edges. 
    The cost of edge $(v_1,v_k)$ is $1+1/(k-2)^2$ while each edge between $v_1$ and $v_i$ ($2\leq i\leq k-1$) has a cost $1/(k-2)^2$. The cost of each dash edge is 1. The terminal set $T$ consists of $v_1,v_2,...,v_k$ while the prediction is $\predT = \{v_1,v_k,v_{k+1},...,v_{2k-2}\}$.}
    \label{fig:oapt_lower_bound}
\end{figure}
To prove Lemma~\ref{lem:lower}, we first construct an instance and then show that algorithm OAPT is $\eta$-competitive on it.
The instance is shown in Fig.~\ref{fig:oapt_lower_bound}.
Due to space, the detailed proof is deferred to Appendix~\ref{sec:omit}.


Next we prove the upper bound of the algorithm's performance. The solution $A$ is partitioned into two sets $A_1$ and $A_2$.  We bound the cost of these sets separately.  The following lemma bounds the cost of $A_1$.  Essentially, these edges do not cost most than $O(\log \eta)$ because there are at most $\eta+1$ terminals that contribute to edges in $A_1$ and their cost is bounded by running the traditional online greedy algorithm on these terminals, which is logarithmically competitive.

\begin{lemma}\label{lem:greedy}
Let $c(A_1)$ be the total cost of edges in $A_1$. We have $c(A_1) \leq O(\log \eta) \opt$, where $\opt$ is the optimal objective value.
\end{lemma}

\begin{proof}
Let $T'$ be the terminals which get connected to the tree via Case (1).
By definition of $\eta$, we have $|T'|\leq \eta +1$. This is because all terminals that are in $T \setminus \predT$ are in Case (1) and these contribute to $\eta$. The only other terminal that executes Case (1) is the first terminal in $\predT$ that arrives.

Consider a new problem instance $\cI'$ where the terminal set is $T'$.  In this new instance, the graph along with edge costs  remains the same. Let $\opt(\cI')$ be the optimal cost of this new instance and $\opt$ the optimal cost of the original problem.  Notice that $\opt(\cI') \leq \opt$ as $T'\subseteq T$ and $\opt$ represents a feasible solution for the new instance.   

We know that the cost of the traditional online greedy algorithm is at most  $O(\log(|T'|)) \opt(\cI')$ because it is $\log (|T'|)$ competitive on the instance $\opt(\cI')$ by Theorem~\ref{thm:greedy}. This holds no matter the arriving order of the terminals.  Let $c(\texttt{Greedy})$ be the cost of the greedy algorithm if the terminals in $T'$ arrive consistent with the arriving order for the original problem. We have that $c(\texttt{Greedy}) \leq O(\log(|T'|))\opt(\cI') \leq O(\log(\eta))\opt(\cI') \leq O(\log(\eta))\opt$.

Notice that the shortest edge from $t_i$ to $T_{i-1}$ is definitely at most the shortest edge from $t_i$ to $T_{i-1}\cap T'$.  Thus, $c(A_1) \leq c(\texttt{Greedy}) \leq O(\log(\eta))\opt$.
\end{proof}

Now we focus on the second set $A_2$.  These terminals potentially cause the bulk of the cost for the algorithm.  We will bound $c(A_2)$ by $O(\eta)\opt$, which proves Theorem~\ref{thm:oapt}.  
We first show that for each edge $e\in A_2$ has cost at most $\opt$. Notice that this is not enough to prove $c(A_2) \leq O(\eta)\opt$  because the number of edges in $A_2$ could be as large as $k-1$. For the remainder of this section, let $P_i$ denote the path added into $A_2$ in iteration $i$. 

\begin{lemma}\label{lem:edgecost}
Each edge in $A_2$ has cost at most $\opt$.
\end{lemma}

\begin{proof}
Consider when terminal $t_i$ arrives, the algorithm executes Case (2) and the path $P_i \neq \emptyset$. Notice that if Case (2) is executed then there is a terminal in $t_j \in \predT \cap T_{i-1}$ that has arrived before $t_i$. Moreover, for any terminal $t_j \in \predT \cap T_{i-1}$, the cost of edge $(t_i,t_j)$ is at most $\opt$ because these two nodes are connected in the optimal solution and $c(t_i,t_j)$ is the minimum cost to connect them. 
To show the lemma, we show that for any edge $e\in P_i$, $c_e \leq c(t_i,t_j)$.  This then bounds the cost of any edge in $A_2$ by $\opt$.

Fix the terminal $t_j \in \predT \cap T_{i-1}$ that $t_i$ connects to using path $P_i$.   If the edge  $(t_i,t_j)$ is in $\mst(\predT)$ then this will be the unique edge in $P_i$. If $(t_i,t_j)$ is not in $\mst(\predT)$ then by Lemma~\ref{lem:cycle} every edge on the cycle $P_i \cup \{(t_i,t_j)\}$ has cost at most $c(t_i,t_j) \leq \opt$.
\end{proof}



We are ready to bound the cost of the edges in $A_2$. 
\begin{lemma}\label{lem:mst}
The edges of $A_2$ can be partitioned into two sets $B_1$ and $B_2$, where $c(B_1) \leq \opt$ and $|B_2| \leq \eta.$ Moreover, the total cost of edges in $A_2$ is at most $O(\eta)\opt$.
\end{lemma}
\begin{proof}

We begin by partitioning the edges of $A_2$ into two sets $B_1$ and $B_2$.  Let $E'$ contain the edges in $\mst(\predT) \cap \mst(\predT \cap T)$.  
Initialize $S = \mst(\predT)$.  The set $S$ will always be a spanning tree of $\predT$. We do the following iteratively. For each edge $e \in   \mst(\predT \cap T) \setminus \mst(\predT)$, we add it to $S$ and remove an arbitrary edge $e'  \in \mst(\predT) \setminus \mst(\predT \cap T)$ from $S$ that forms a cycle. The removed edge $e'$ is added to $E'$.  
Set $B_1 = E' \cap A_2$ and $B_2 = A_2 \setminus E'$.




    






Intuitively, the above procedure  obtains a spanning tree $S$ of $\predT$ by replacing some edges in $\mst(\predT)$ that are not in $\mst(\predT) \cap \mst(\predT \cap T)$ with the edges in $\mst(\predT\cap T)$.  We have that $c(E') \leq c(\mst(\predT \cap T))$.  This is because $c(e) \geq c(e')$ in each step of the algorithm by definition of $\mst(\predT)$ and Lemma~\ref{lem:cycle}.  Knowing $c(\mst(\predT \cap T)) \leq  \opt$, we see that $c(B_1) \leq c(E') \leq \opt$.


According to the algorithm, the number of edges in $E'$ is exactly the same as the number of edges in $\mst(\predT \cap T)$. In other words, $|E'| = k-\eta -1$ and $|\mst(\predT) \setminus E'| = \eta$. Since $A_2$ is a subset of $\mst(\predT)$, $|B_2| \leq |\mst(\predT) \setminus E'| = \eta$. Namely, the number of edges in the second partition is at most $\eta$. Using Lemma~\ref{lem:edgecost}, we have $c(B_2) \leq \eta \opt$, completing the proof of this lemma.

\end{proof}

\begin{proof}[\textbf{Proof of Theorem~\ref{thm:oapt}}]
The theorem can be proved directly by Lemma~\ref{lem:greedy} and Lemma~\ref{lem:mst}: 
$c(A) \leq c(A_1) + c(A_2) \leq O(\log \eta)\opt + O(\eta)\opt = O(\eta)\opt.$ The lower bound in the theorem is given in Lemma~\ref{lem:lower}. Altogether, we have the main theorem. 
\end{proof}

\subsection{An Improved Online Algorithm Leveraging Predictions}

In this section, we will build on the simple algorithm to give a more robust online algorithm that has a competitive ratio of $O(\log \eta)$.  Notice that in the prior proof, the large cost arises due to the edges that are added in Case (2),  especially the edges in $B_2 = A_2\setminus E'$ in proof of the final lemma.   The new algorithm is designed to mitigate this cost.

\smallskip
\noindent \textbf{Improved Online Algorithm with Predicted Terminals (IOAPT):} Let $\predT$ be the predicted set of terminals and $\mst(\predT)$ be the minimum spanning tree on $\predT$.  Let $T_{i}$ be the first set of $i$ terminals that arrive online.  $T_k$ contains all online terminals. 

Initialize $A = \emptyset$ to be the subgraph that the algorithm will construct connecting the online terminals. The algorithm returns the set of edges in $A$ after all terminals arrive.  We divide the edges of $A= A_1 \cup A_2$ into two sets, $A_1$ and $A_2$ depending on the case that causes us to add edges to $A$.    Consider when terminal $t_i$ arrives.

\begin{itemize}
\item \textbf{Case 1:}  If $t_i \notin \predT$ or $t_i$ is the first terminal in $\predT$ to arrive, add to $A_1$ the shortest edge in $G$ connecting $t_i$ to terminals $T_{i-1}$ that have arrived. No edge is bought if this is the first terminal that arrives. 
\item \textbf{Case 2:} Otherwise, find the shortest path $P_i$ connecting $t_i$ to $\predT \cap T_{i-1}$ in $\mst(\predT)$.  
Use $e_i$ to denote the shortest edge connecting $t_i$ to $\predT \cap T_{i-1}$ in $G$.
We add to $A_2$ a sub-path $P_i' \subseteq P_i$ such that its endpoints contain $t_i$ while its total cost is in $[c_{e_i},2c_{e_i}]$. 
Next, add $e_i$ to $A_2$ if $t_i$ is not connected to the tree after adding $P'_i$. 
\end{itemize}



Notice that in Case 2, we can always find such a sub-path $P_i'$ due to the property of the minimum spanning tree and the assumption that $G$ is a metric. Thus, the algorithm always computes a feasible solution. We have the following two lemmas.  The proofs are identical to Lemma~\ref{lem:edgecost} and Lemma~\ref{lem:mst} respectively.

\begin{lemma}\label{lem:iedgecost}
The cost of any edge in $A_2$ computed by IOAPT is at most $\opt$. 
\end{lemma}


\begin{lemma}\label{lem:partition}
The edges of $A_2$ can be partitioned into two sets $B_1$ and $B_2$ where $c(B_1) \leq \opt$ and $|B_2| \leq \eta.$
\end{lemma}

With these lemmas, we can prove the theorem.

\begin{theorem}\label{thm:ioapt}
The competitive ratio of  IOAPT is $O(\log \eta)$. 
\end{theorem}
\begin{proof}
The analysis of $c(A_1)$ is the same as that in OAPT.  The proof of Lemma~\ref{lem:greedy} immediately implies $c(A_1) \leq O(\log \eta)\opt$. Next we focus on bounding the cost of $A_2$.

Let $\newdelta_{i} c(A_2)$ be the increase in $c(A_2)$ in when terminal $t_i$ arrives. According to definition of the algorithm, we know  $\newdelta_{i}c(A_2) \leq 2c(P_i')$ and $\newdelta_{i}c(A_2) \leq 3c_{e_i}$.  Lemma~\ref{lem:partition} states the edges in $\mst(\predT)$ can be partitioned into two sets $E_0$ and $E_1 := \mst(\predT) \setminus E_0$, where the cost of $E_0$ is at most $2\opt$ and the number of edges in $E_1$ is $\eta$. 

Let $T^G$ be the `good' terminals that execute Case (2) and $P_i' \subseteq E_0$.  Let $T^B$ be the remaining `bad' terminals. We see the following for the good terminals, 
$c(A_2) = \sum\limits_{i\in T^G} \newdelta_i c(A_2) \leq \sum\limits_{i\in T^G}  2c(P_i') \leq 2c(E_0) \leq O( \opt).$ In other words, if the sub-path added in each iteration always belongs to $E_0$, the total cost of $c(A_2)$ is bounded by a constant factor of $\opt$. Say an iteration is good if the sub-path added in it belongs to $E_0$. The total increment of all good iterations is at most $O(\opt)$.

We use the second upper bound to analyze the cost of the bad terminals. This follows similarly to the proof of Lemma~\ref{lem:greedy}. Indeed, we know the following,  $\sum\limits_{i\in T^B} \newdelta_i c(A_2) \leq \sum\limits_{i\in T^B} 3c_{e_i}.$
If iteration $i$ is bad, there exists at least one edge in sub-path $P_i'$ belonging to $E_1$. Since $|E_1| = \eta$, the number of bad iterations is at most $\eta$.  The total cost of these iterations $\sum\limits_{i\in T^B} 3c_{e_i}$ is at most $3$ multiplied by the cost of running the greedy algorithm on the terminals in $T^B$. Let $\opt(T^B)$ be the optimal solution on $T^B$.  We know that $\opt \geq \opt(T^B)$.  Moreover, we know that the greedy algorithm has cost at most $O(\log(|T^B|)) \opt(T^B) \leq O(\log(|T^B|)) \opt \leq O(\log \eta)\opt$.  Thus we have the following,  
$\sum\limits_{i\in T^B} \newdelta_i c(A_2) \leq O(\log \eta) \opt$. This completes the proof of Theorem~\ref{thm:ioapt}.
\end{proof}

The competitive ratio $O(\log (\eta))$  approaches the worst-case bound $\log (k)$ when $\eta = k$. Here we give a stronger statement to show our algorithm optimally uses the predictions. The proof is deferred to Appendix~\ref{sec:omit}.

\begin{theorem}\label{thm:undirected_lower_bound}
For online undirected Steiner tree with predicted terminals, given any $\eta \geq 1$, no online algorithm has a competitive ratio better than $\Omega(\log(\eta))$.
\end{theorem}

\noindent \textbf{Improving the Performance of the Algorithm in Practice.} We describe a practical modification of the algorithm. This modification ensures that the algorithm maintains its theoretical bound, while improving the performance. The observation is that the algorithm may purchase edges not needed for feasibility. Some edges added by our algorithm  are purchased based on  predicted terminals and  they will become useless if these predicted terminals do not arrive. We can choose not to buy these edges immediately.  When $t_i$ arrives, the edges in $P'_i$ are not bought immediately if we buy $e_i$.   Instead, the algorithm buys the edges the first time a terminal uses them to connect to previous terminals.

    


\section{Online Steiner Tree in Directed Graphs}\label{sec:directed_case}

This section considers online Steiner tree when the graph is directed. The input is  a directed graph $G=(V,E)$, where each edge $e$ has cost $c_e\geq 0$, a root vertex $r\in V$ and a terminal set $T\subseteq V$ that arrives online.   This paper assumes without loss of generality, that $c_e > 1$ for any edge $e$. Additionally the input graph is assumed to ensure that there exists a directed path from root $r$ to every vertex in $V$.

The terminals in $T$ arrive online.  When a terminal $v \in T$ arrives the algorithm must buy some edges to ensure there is a directed path from the root $r$ to $v$ in the subgraph induced by the bought edges. The goal is to minimize the total cost of the bought edges.   

In directed graphs, the worst-case bound on the competitive ratio is $\Omega(k)$  \cite{DBLP:journals/ipl/WestbrookY95}. Our main result shows that  we can break through this bound.

\begin{theorem}\label{thm:directed_main}
Given a predicted terminal set $\predT$, there exists an algorithm with competitive ratio at most $O(\log k + \eta)$, where $\eta:=k-|T\cap \predT|$.
\end{theorem}

The algorithm claimed in Theorem~\ref{thm:directed_main} is $O(\log k)$-consistent and $O(k)$-robust, meaning that the ratio is $O(\log k)$ if $\eta=0$ and is at most $O(k)$ for any $\eta$.  The algorithm for directed graphs builds on the algorithm for undirected graphs. As before, there are two sets of edges $A_1, A_2$. The set $A_1$ contains edges that are bought because a terminal arrives that was not predicted. As in the undirected case such these edges are bought using a greedy algorithm.  The edges in $A_2$ are bought using a different algorithm over the undirected case. 


\smallskip
\noindent \textbf{Online Algorithm with Predicted Terminals in Directed Graphs:} 
Initialize $\lambda = 1$ to be a parameter, which is intuitively a guess of the maximum connection cost of any terminal in $T$. Let $\predT(\lambda) := \{t\in \predT  \;\; | \;\; c(t,r) \leq \lambda \}$ be the set of predicted terminals that have a path to the root of cost at most $\lambda$.   Let $\mdst(\predT(\lambda))$ be the minimum directed Steiner tree of $\predT(\lambda)$, which can be computed by an offline optimal algorithm\footnote{Noting that this problem is NP-hard and  it is known to be inapproximable within a $O(\log k)$ ratio unless $P=NP$~\cite{DBLP:conf/stoc/DinurS14}, we do not have efficient optimal algorithms or approximation algorithms in practice. Thus, the directed case is more for theoretical interests. }. 

Initialize $A_1 = \emptyset$ and $A_2 = \emptyset$. The edges that are bought will be $A = A_1 \cup A_2$. Order the terminals such that $t_i$ arrives before $t_{i+1}$ and let $T_i = \{t_1, t_2, \ldots, t_{i}\}$ be the first $i$ terminals to arrive. Let $\beta_i = \max_{t_j \in T_i} c(t_j,r)$ be the maximum cost of connecting a terminal in $T_i$ directly to the root. 

Consider when a terminal $t_i \in T$ arrives. If $\beta_i > \lambda$ then both  increase $\lambda$ by a factor 2 and  update $\predT(\lambda)$ and $\mdst(\predT(\lambda))$.  Next perform one of the following. 

\begin{itemize}
    \item  If $t_i \notin \predT(\lambda)$ then add the shortest path from $t_i$ to $r$ to $A_1$, buying these edges.
    \item Otherwise, add the unique path from $t_i$ to $r$ in $\mdst(\predT(\lambda))$ to $A_2$.
\end{itemize}

Our goal is to show the following theorem.

\begin{theorem}\label{thm:directed_oapt}
The competitive ratio of the  Algorithm for directed Steiner tree is $O(\log k + \eta)$.
\end{theorem}

Before proving the theorem, we show a technical lemma.

\begin{lemma}\label{lem:directed_oapt}
For any $\lambda$, $c(\mdst(\predT(\lambda))) \leq \opt + \lambda \eta$.
\end{lemma}
\begin{proof}
The proof idea is to construct a feasible Steiner tree of $\predT(\lambda)$ whose value is at most $\opt + \lambda \eta$. Then the inequality will hold due to the optimality of $\mdst(\predT(\lambda))$. The feasible tree is constructed as follows: connect all terminals in $\predT(\lambda) \cap T$ to the root in the same way as the optimal solution and add the shortest path from $t$ to $r$ for each terminal $t\in \predT(\lambda) \setminus T$. The total cost of the former part is at most $\opt$ while the latter term incurs a cost of $\lambda \eta$ since $c(t,r) \leq \lambda$ for any terminal $t\in \predT(\lambda)$. Thus, the total cost of this subgraph is at most $\opt + \lambda \eta$, implying that $c(\mdst(\predT(\lambda))) \leq \opt + \lambda \eta$.
\end{proof}

We can now prove the main theorem. Due to space see  Appendix~\ref{sec:omit}. 



\section{Experimental results}\label{sec:exp}
\begin{figure*}[t]
\centering
\subfigure[]{
    \centering
    \includegraphics[width=0.3\textwidth]{ 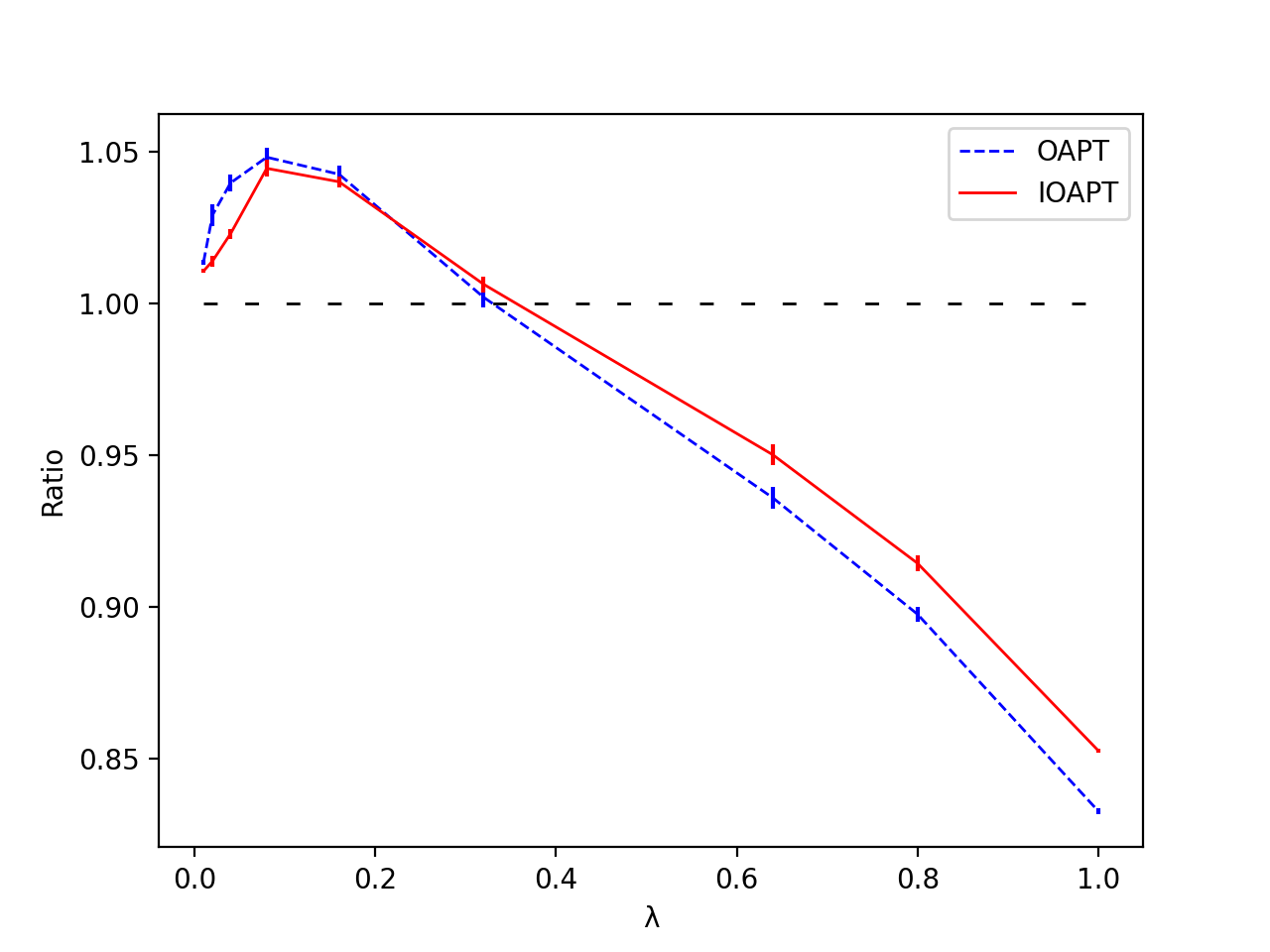}
    \label{fig:random_robust}
    }
\subfigure[]{
    \centering
    \includegraphics[width=0.3\textwidth]{ 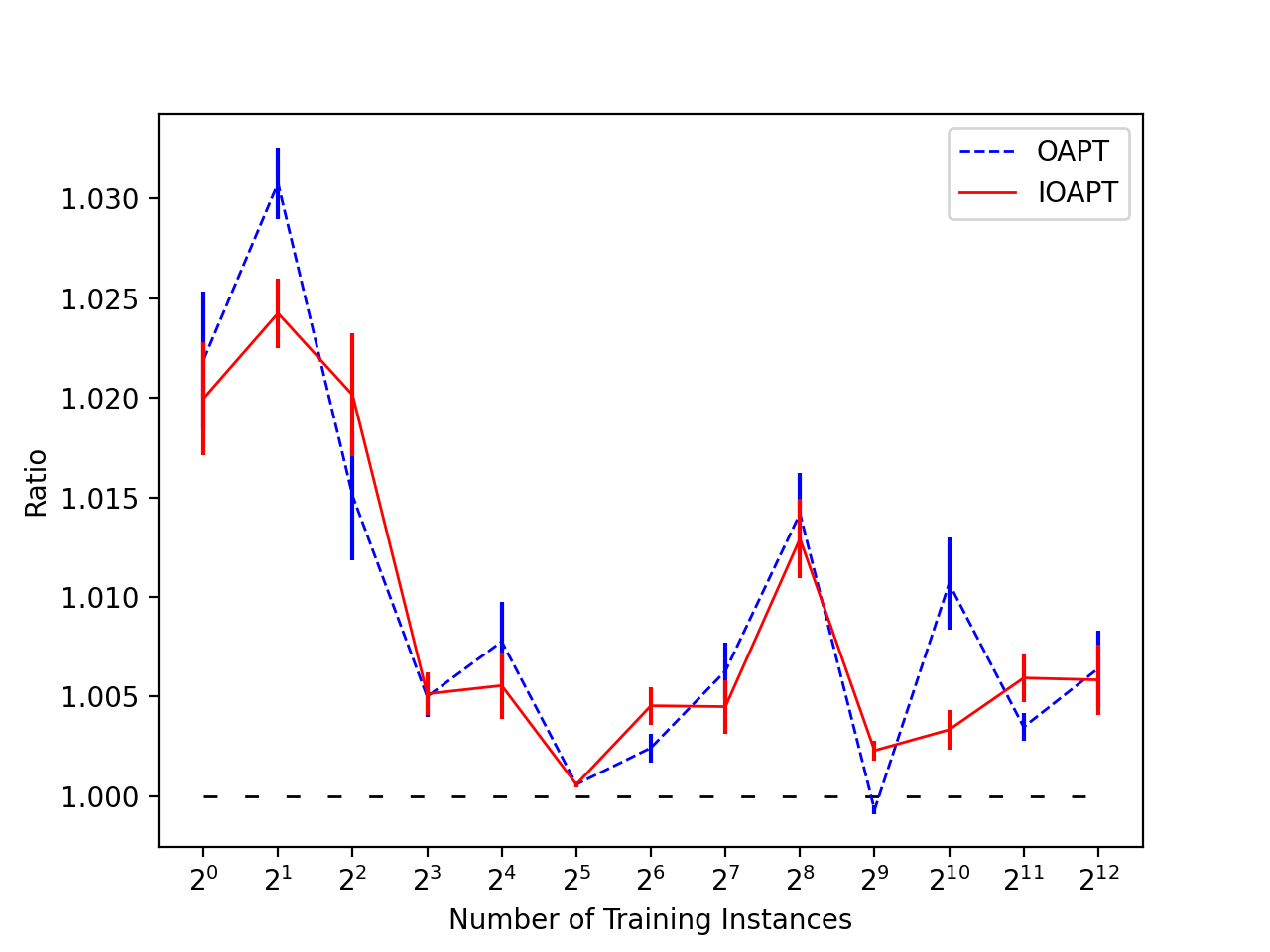}
    \label{fig:random_uniform}
    }
\subfigure[]{  
    \centering
    \includegraphics[width=0.3\textwidth]{ 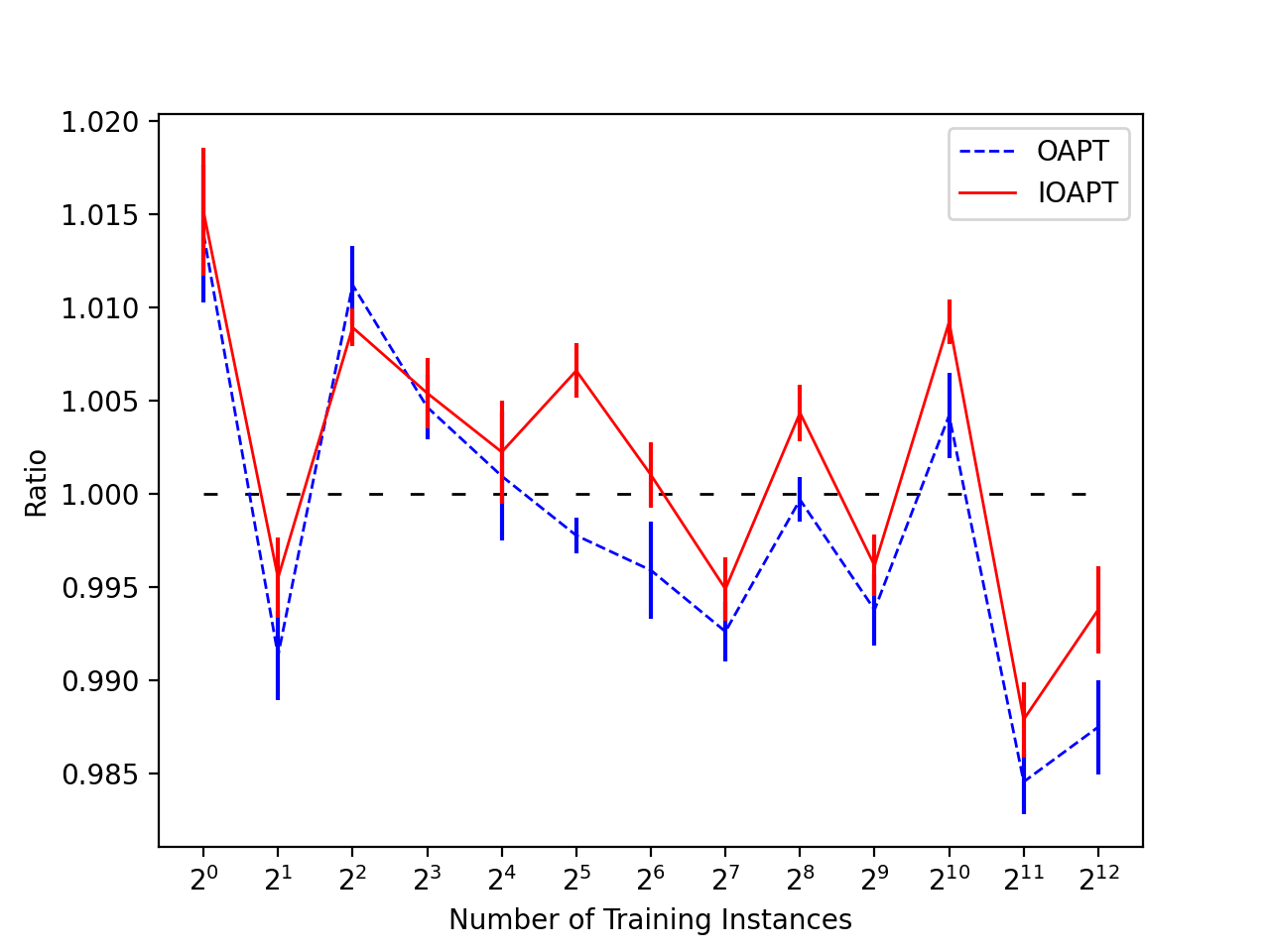}
    \label{fig:random_good_distri}
    }

    \caption{The experimental results on random graphs. The ratio is the algorithm's performance relative to the baseline.  Fig.~\ref{fig:random_robust} shows the performance of algorithms over different $\lambda$'s, corresponding to the robustness experiment. 
    Fig.~\ref{fig:random_uniform} and Fig.~\ref{fig:random_good_distri} are, respectively, the algorithms' performance over the number of training instances on the uniform distribution and the two-class distribution. Their corresponding prediction errors are deferred to Appendix~\ref{sec:prediction_error}. Note that some of the $x$-axes are on log-scale.
    }
    \label{fig:random_graph}
\end{figure*}

\begin{figure*}[t]
\centering
\subfigure[]{
    \centering
    \includegraphics[width=0.3\textwidth]{ 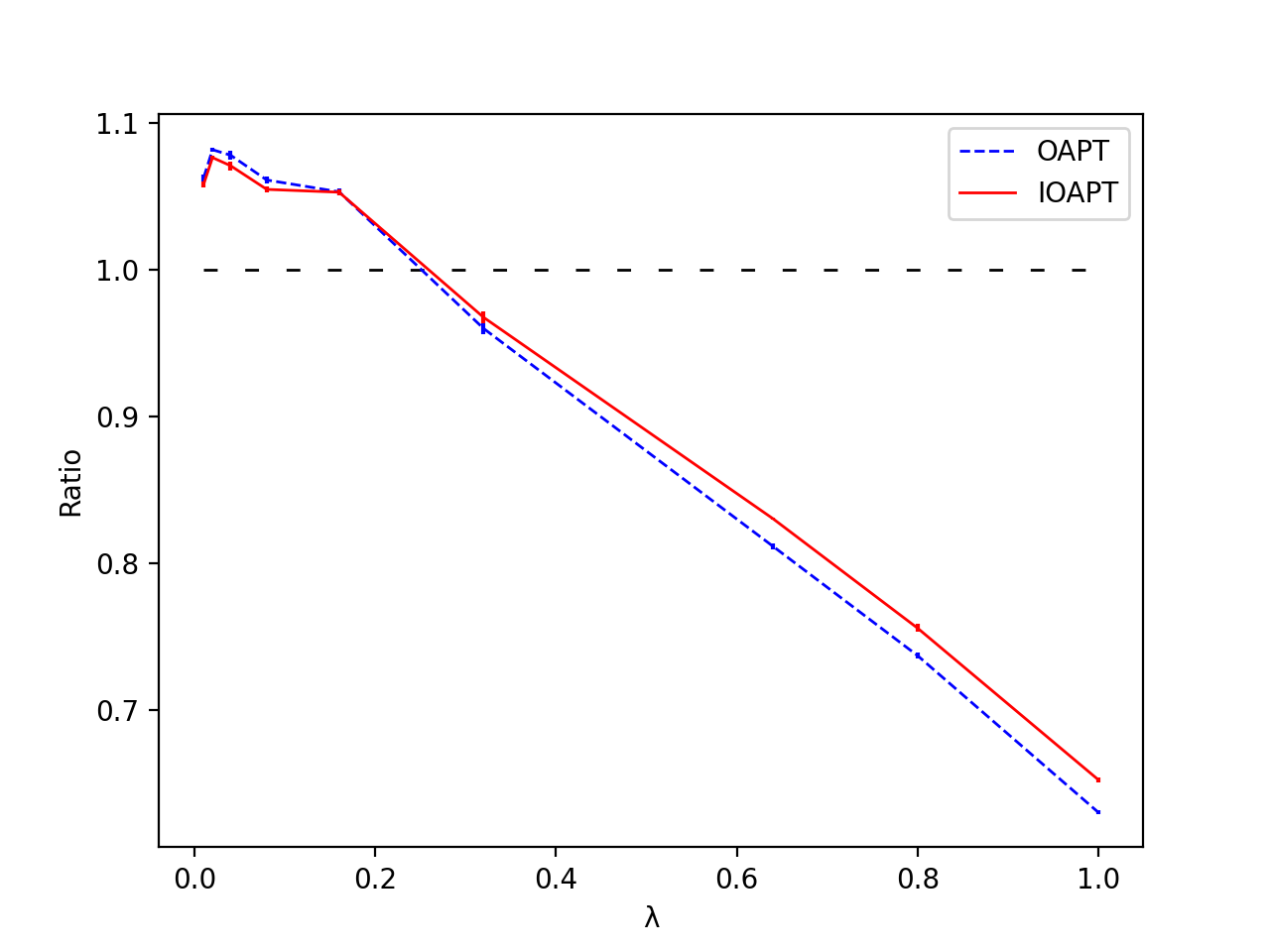}
    \label{fig:road_robust}
    }
\subfigure[]{
    \centering
    \includegraphics[width=0.3\textwidth]{ 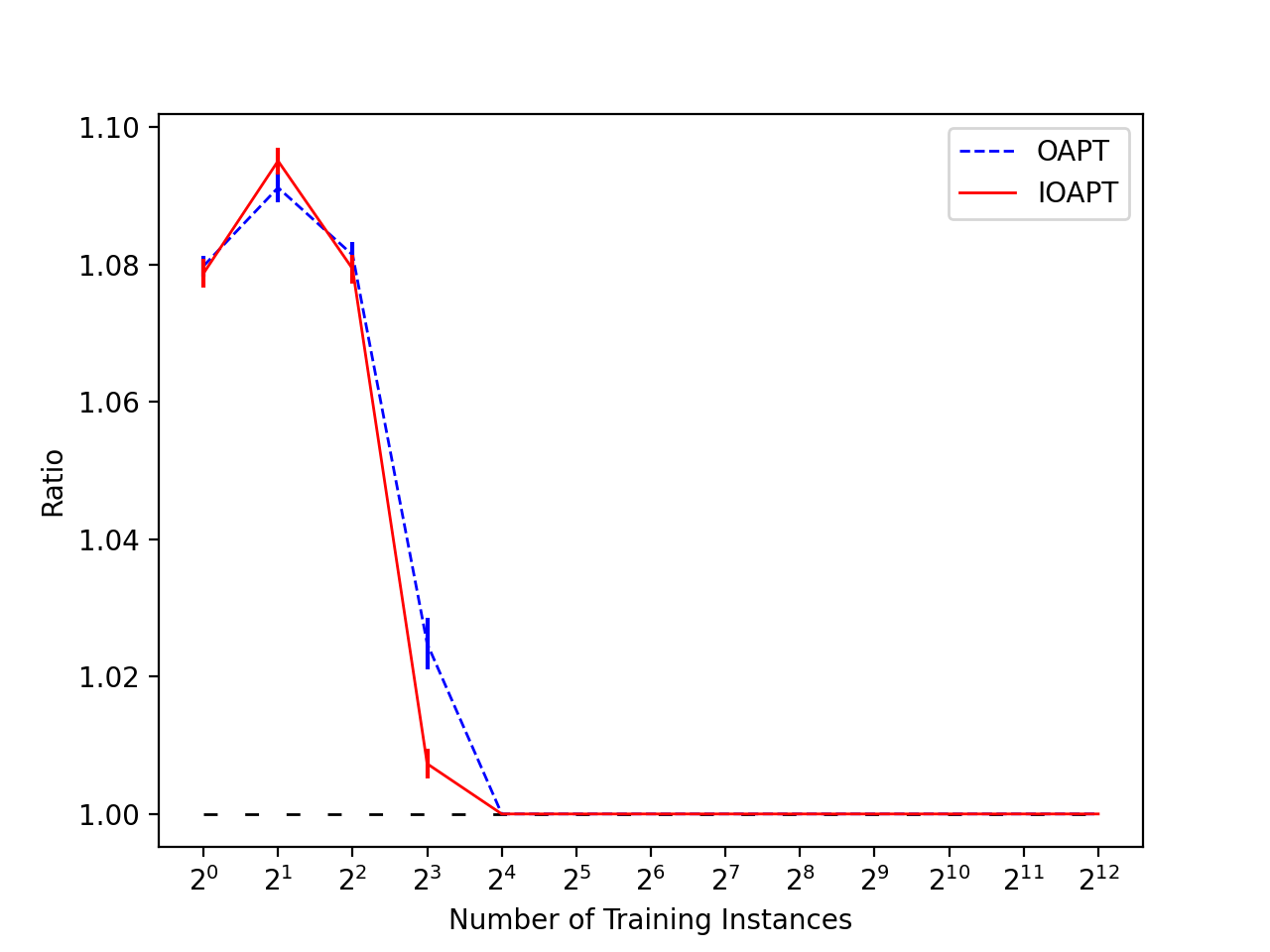}
    \label{fig:road_cluster_x10}
    }
\subfigure[]{
    \centering
    \includegraphics[width=0.3\textwidth]{ 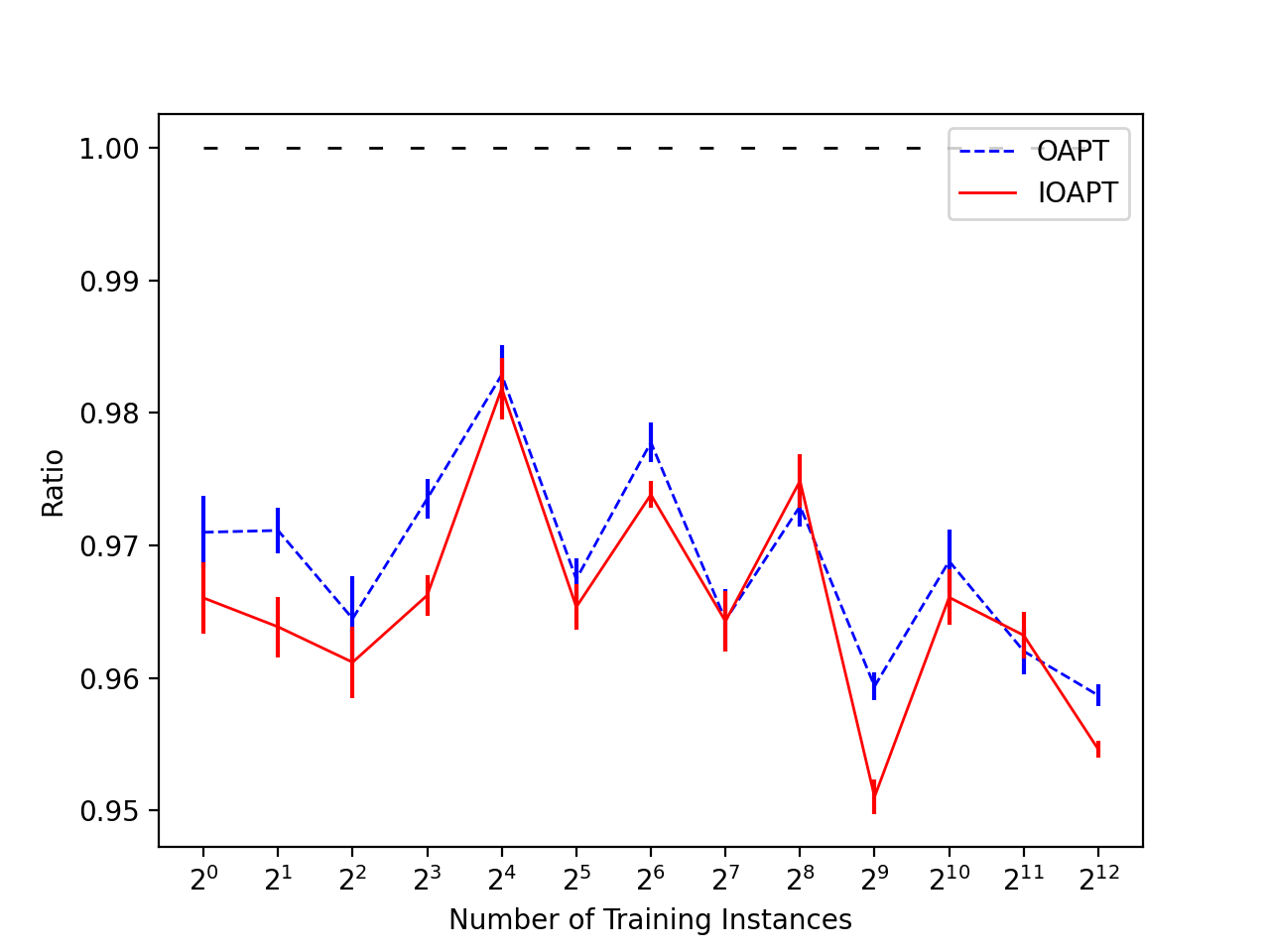}
    \label{fig:road_cluster_x100}
    }
    \caption{Results on road graphs. Fig.~\ref{fig:road_robust} shows the performance of algorithms over different $\lambda$'s.  Fig.~\ref{fig:road_cluster_x10} and Fig.~\ref{fig:road_cluster_x100} are, respectively, the algorithms' performance and prediction error over different numbers of training instances when $x=10$ and 100. The corresponding prediction errors are deferred to Appendix~\ref{sec:prediction_error}. }
    \label{fig:road_graph}
\end{figure*}

This section investigates the empirical performance of the proposed algorithm OAPT and IOAPT for the undirected Steiner tree problem. 
The goal is to answer the following two questions:
\begin{itemize}
    \item Robustness - How much prediction accuracy does the algorithms need to outperform the baseline algorithm empirically? 
    \item Learnability - How many samples are required to empirically learn predictions sufficient for the algorithms to perform better than the baseline? 
\end{itemize}

The baseline we compare against is the online greedy algorithm which is the best traditional online algorithm. 
We investigate the performance of both OAPT and IOAPT.



\subsection{Setup} 
The experiments\footnote{The code is available at \url{https://github.com/Chenyang-1995/Online-Steiner-Tree}} are conducted on a machine running Ubuntu 18.04 with an i7-7800X CPU and 48 GB memory.  Experiments are averaged over 10 runs. We consider two types of graphs. 


\paragraph{Random Graphs.}
The number of nodes in a graph is set to be 2,000 and 50,000 edges are selected uniformly at random. The cost of each edge is an integer sampled uniformly from $[1 , 1000]$. To ensure the connectivity of graphs, we add all remaining edges to form a complete graph, given the edges high cost of 100,000.

\paragraph{Road Graphs.}
The road network of Bay Area is provided by The 9th DIMACS Implementation Challenge\footnote{http://users.diag.uniroma1.it/challenge9/download.shtml} in graph format where a node denotes a point in Bay Area and the cost of an edge is the road length between the two endpoints. The original data contains 321,270 nodes and 400,086 edges. In the experiments, we employ the same sampling method as in~\cite{DBLP:conf/aistats/MoseleyVW21} to sample connected subgraphs from this large graph. 
Briefly, we draw rectangles with a certain size on the road network randomly and construct a subgraph from each rectangle. 
The experiments employ 4 sampled subgraphs with $23512 \pm 1135$ nodes and $ 31835 \pm 1815$ edges. These graphs give the same trends, thus, we show one such graph in the main body and others appear Appendix~\ref{sec:different_graph}. 

The terminal set $T$ and the prediction $\predT$ are constructed differently depending on the experiments. 

\subsection{Robustness to Accuracy} 
This experiment tests the performance of OAPT and IOAPT when the  prediction accuracy varies. Recall that $k$ is the number of terminals. We set $k=200$ and $2,000$ respectively for random graphs and road graphs unless stated otherwise\footnote{We also conduct experiments with different numbers of terminals. The results are present in Appendix~\ref{sec:more_terminals}.}.  The set $T$ of $k$ terminals are sampled uniformly from the vertex set $V$.  They arrive in random order. 


We now construct the predictions. Let $\lambda \in [0,1]$ be a parameter corresponding to the prediction accuracy. 
First, we sample a node set $\predT_0$ with $k\lambda$ nodes uniformly from the terminal set $T$. And then another node set $\predT_1$ with $k(1-\lambda)$ nodes is sampled uniformly from the non-terminal nodes $V\setminus T$. Let $\predT_0 \cup \predT_1$ be the predicted terminal set $\predT$. 
Notice that $\lambda$ indicates the prediction accuracy. Thus, testing the performance of algorithms with different $\lambda$'s answers the robustness question.  This experiment is  in Fig.~\ref{fig:random_robust} and~\ref{fig:road_robust}.


\subsection{Learning the Terminals} 

Here we construct instances where the algorithm explicitly learns the terminals. Each such instance will have a distribution over terminal sets of size $k$ and employ random order.  
We will sample $s$ training instances of $k$ terminals $T_1, T_2, \ldots T_s$. The learning algorithm used to predict terminals is defined as follows.

\paragraph{The Learning Algorithm.}
A node $v$ is predicted to be in $\predT$ with probability $f(v)/s$ if $f(v) > \theta s$, where $f(v)$ is the number of sampled sets in which node $v$ appears and $\theta$ is a parameter in $[0,1]$.
Note that the number of predicted terminals may not equal $k$.

There is a question on how to choose $\theta$.  This is done as follows.  We choose an instance $T_i$ from the training set at random and check which $\theta$ would give OAPT (IOAPT) the best performance on this instance.  We then use this $\theta$ for OAPT (IOAPT) on the online instance.  For efficiency, we only consider $\theta\in \{0,0.2,0.4,0.6,0.8,1\}$.

\paragraph{Distribution for Random Graphs.}
Two distributions are considered for random graphs. The first is
 a bad distribution where there is nothing to learn, the uniform distribution.  In this case, all terminals are drawn uniformly from $V$.   
 The second is called a two-class distribution where there is a set of nodes to learn. Let  $V_h$  be a small collection of nodes that will be terminals with higher probability.  $V_h$ is set to $400$ nodes uniformly at random. Let $k=200$ be the number of terminals.  Half are drawn from $V_h$ and half from $V \setminus V_h$. Here we hope the learning algorithm quickly learns $V_h$, and further, our algorithms can take advantage of the predictions.
 The results appear in Fig.~\ref{fig:random_uniform} and~\ref{fig:random_good_distri}. 





\paragraph{Distribution for Road Graphs.}  This experiment is designed to model the case where terminals can appear in geographical similar locations. The graph will be clustered and a specified number of terminals will arrive per cluster following a distribution over nodes in the cluster. 
Use $r$ to denote the radius of the graph. Given a parameter $\sigma$, partition all nodes into several clusters such that the radius of each cluster is at most $\sigma r$. The greedy clustering algorithm~\cite{DBLP:journals/tcs/Gonzalez85} is used and described in Appendix~\ref{sec:code}.  We let $\sigma=0.1$ in the experiments unless state otherwise. The terminal set $\predT$ is obtained by picking $\lfloor 2000/x \rfloor$ random clusters and sampling $x$ terminals uniformly from each selected cluster.
We let $x$ be 10 and 100. When $x=10$ the distribution is harder to learn than when $x=100$. See Fig.~\ref{fig:road_cluster_x10} and~\ref{fig:road_cluster_x100} for the results.

\subsection{Empirical Discussion} We see the following trends.
\begin{itemize}
    \item Both Fig.~\ref{fig:random_robust}~\ref{fig:road_robust} show that the algorithms perform well on different graphs even with modestly correct predictions.  Once about $30\%$ of the predictions are correct, the algorithms perform better than the baseline.
    
    \item Fig.~\ref{fig:random_uniform}~\ref{fig:road_cluster_x10} show the algorithms are robust for difficult distributions, which are sparse distributions where there is effectively nothing to learn. The learning algorithm will quickly realize that  predictions cause negative effects and then output very few predicted terminals (see prediction errors in Appendix~\ref{sec:prediction_error} for more corroborating experiments). After tens of training instances, the ratios become never worse than 1.01. 
    \item  
    Fig.~\ref{fig:random_good_distri}~\ref{fig:road_cluster_x100} show the learning algorithm quickly learns good distributions. Further, both online algorithms have strong performance using the predictions.  We conclude that with a small number of training samples, the learning algorithm is able to learn useful predictions sufficient for the online algorithms to outperform the baseline.  
\end{itemize}

These experiments corroborate the theory.  The algorithms obtain much better performance than the baseline even with modestly good predictions. If
given very inaccurate predictions, the algorithms are barely worse than the baseline. Moreover, we see that only a small number of sample instances are needed for the algorithms to have competitive performance when terminals  arrive from a good distribution.



\section{Conclusion}
Online Steiner tree is one of the most fundamental online network design problems. It is a special case or a sub-problem of many online network design problems. Steiner tree captures the challenge of building networks online and, moreover, Steiner tree algorithms are often used as building blocks or subroutines for more general problems. As the community expands the learning augmented algorithms area into more general online network design problems, this paper provides models, and algorithmic and analysis techniques that can be leveraged for these problems.

\newpage
\section*{Acknowledgements}
Chenyang Xu was supported in part by Science and Technology Innovation 2030 –"The Next Generation of Artificial Intelligence" Major Project No.2018AAA0100902.
Benjamin Moseley was supported in part by NSF grants  CCF-1824303, CCF-1845146, CCF-2121744, CCF-1733873 and CMMI-1938909.  Benjamin Moseley was additionally supported in part by a Google Research Award, an Infor Research Award, and a Carnegie Bosch Junior Faculty Chair.
We thank Yuyan Wang for sharing their experimental data and thank the anonymous reviewers for their insightful comments and suggestions. Further, we thank Mirko Giacchini for the discussion on some implementation details of the proposed algorithms.

\bibliography{ref}

\newpage

\appendix
\section{Pseudo-codes for Algorithms}\label{sec:code}

This section presents the pseudo-codes of algorithms mentioned in previous sections: algorithm OAPT, algorithm IOAPT, the algorithm for directed graphs, the learning algorithm and the clustering algorithm (a natural variant of the greedy algorithm given in~\cite{DBLP:journals/tcs/Gonzalez85}). 

\begin{algorithm}[htbp]
\caption{Online algorithm with predicted terminals (OAPT)}
\label{algo:oapt}
\KwIn{The graph $G=(V,E)$, the edge cost function $c$, the terminal set $T$ which shows up online and the predicted terminal set $\predT$ }

Compute a minimum spanning tree $\mst(\predT)$ of the prediction $\predT$.

Define our solution $A:=A_1 \cup A_2$ and initially, $A_1,A_2\leftarrow \emptyset.$

\While{a terminal $t_i$ arrives}
{
    Let $T_{i-1}$ be the set of terminals arriving till the last iteration.
    
    \If{$t_i\notin \predT$ or $t_i$ is the first arriving vertex in $\predT$}
    {
        Add the shortest edge connecting $t_i$ to $T_{i-1}$ in $G$ into $A_1$.
    }
    \Else
    {
        Add the shortest path connecting $t_i$ to $\predT \cap T_{i-1}$ in $\mst(\predT)$ into $A_2$.
    }
}
\KwOut{Edge set $A$}
\end{algorithm}

\begin{algorithm}[htbp]
\caption{Improved online algorithm with predicted terminals (IOAPT)}
\label{algo:ioapt}
\KwIn{The graph $G=(V,E)$, the edge cost function $c$, the terminal set $T$ which shows up online and the predicted terminal set $\predT$ }
Compute a minimum spanning tree $\mst(\predT)$ of the prediction $\predT$.

Define our solution $A:=A_1 \cup A_2$ and initially, $A_1,A_2\leftarrow \emptyset.$

\While{a terminal $t_i$ arrives}
{
    Let $T_{i-1}$ be the set of terminals arriving till the last iteration.
    
    \If{$t_i\notin \predT$ or $t_i$ is the first arriving vertex in $\predT$}
    {
        Add the shortest edge connecting $t_i$ to $T_{i-1}$ in $G$ into $A_1$.
    }
    \Else
    {
        Add the shortest edge $e_i$ connecting $t_i$ to $\predT \cap T_{i-1}$ in $G$ into $A_2$.
        
        Find the shortest path $P_i$ connecting $t_i$ to $\predT \cap T_{i-1}$ in $\mst(\predT)$.
        
        Find a sub-path $P_i' \subseteq P_i$ such that its endpoints contain $t_i$ while its total cost is in $[c_{e_i},2c_{e_i}]$.
        
        Add path $P_i'$ into $A_2$.
    }
}
\KwOut{Edge set $A$}
\end{algorithm}


\begin{algorithm*}[htbp]
\caption{Online algorithm with predicted terminals on directed graphs}
\label{algo:directed_oapt}
\KwIn{The graph $G=(V,E)$, the edge cost function $c$, the terminal set $T$ which shows up online and the predicted terminal set $\predT$ }

Define our solution $A:=A_1 \cup A_2$ and initially, $A_1,A_2\leftarrow \emptyset.$

Initialize a parameter $\lambda$ with 1 and define $\predT(\lambda) := \{t\in \predT | c(t,r) \leq \lambda \}$, where $c(t,r)$ is the cost of the shortest path from $t$ to $r$.

Let $\mdst(\predT(\lambda))$ be the minimum directed Steiner tree of $\predT(\lambda)$, which can be computed by an offline optimal algorithm.

\While{a terminal $t_i$ arrives}
{
    Let $T_i$ be the terminals arriving so far and define $\beta_i := \max_{j\leq i} c(t_j,r)$.
    
    \If{$\beta_i > \lambda$}
    {
        $\lambda \leftarrow 2\lambda$.
        
        Update $\predT(\lambda)$ and $\mdst(\predT(\lambda))$.
    }
    
    \If{$t_i\notin \predT$}
    {
        Add the shortest path from $t_i$ to $r$ into $A_1$.
    }
    \Else
    {
        Add the unique path from $t_i$ to $r$ in $\mdst(\predT(\lambda))$ into $A_2$.
    }
    
}

\KwOut{Edge set $A$}
\end{algorithm*}

\begin{algorithm*}[htbp]
\caption{Learning the terminals}
\label{algo:learning_algo}
\KwIn{The graph $G=(V,E)$, $s$ training instances and an augmenting algorithm $A$ }
\SetKwFunction{Pred}{\textbf{ P}}
	\SetKwProg{Fn}{Function}{:}{}
	\Fn{\Pred{$\theta$}}{
		$\predT \leftarrow \emptyset$
		
		\ForEach{$v\in V$}
		{
		    Let $f(v)$ be the number of instances where $v$ is a terminal.
		    
		    \If{$f(v) > \theta s$}
		    {
		        Add $v$ into $\predT$ with probability $f(v)/s$.
		    }
		}
		
	\textbf{return} $\predT$	
	}

\vspace{2mm}

\textbf{End Function}

Sample an instance $\cI$ uniformly from training instances.

\ForEach{$\theta \in \{0,0.2,0.4,0.6,0.8,1\}$}
{

Use\Pred{$\theta$} to sample a predicted terminal set $\predT(\theta)$.

Compute $A(\cI,\predT(\theta))$, which is the objective value of algorithm $A$ augmented with prediction $\predT$ on instance $\cI$.
}

Pick the $\predT(\theta)$ with the best $A(\cI,\predT(\theta))$ and let it be $\predT$.

\KwOut{Predicted terminal set $\predT$}
\end{algorithm*}

\begin{algorithm*}[htbp]
\caption{A variant of greedy clustering algorithm}
\label{algo:cluster}
\KwIn{The graph $G=(V,E)$, the edge cost function $c$ and a threshold $\sigma$ }

Use $r$ to denote the radius of graph $G$.

Maintain two node sets $A$ and $C$. Set $A\leftarrow V$ and $C\leftarrow \emptyset$ initially.

\While{$A$ is non-empty}
{
    For a node $v$, define $d(v,C) := \min_{w\in C}c(v,w)$, where $c(v,w)$ is the minimum edge cost connecting $v$ and $w$. Let $d(v,C) = 0$ if $C=\emptyset$
    
    Pick a node $v\in A$ with maximum $d(v,C)$.
    
    Add $v$ into $C$: $C\leftarrow C \cup \{v\}$.
    
    Remove all nodes $u\in A$ with $c(v,u) \leq \sigma r$ from $A$ and let them belong to the cluster centered at $u$.  
    
}
\end{algorithm*}
\newpage

\section{Omitted Proofs}
\label{sec:omit}

\begin{proof}[\textbf{Proof of Lemma~\ref{lem:lower}}]
The proof follows by construction of an instance of online Steiner tree where OAPT has cost $\eta \opt$.

The instance is shown in Fig.~\ref{fig:oapt_lower_bound_restate}. There is an edges $v_1v_i$ of cost $\frac{1}{(k-2)^2}$ for $i \in \{2,3,\ldots k-1\}$. The edge $v_1v_k$ has cost $1+\frac{1}{(k-2)^2}$.  There is a cycle $v_1, v_k, v_{k+1}, \ldots, v_{2k-2}, v_1$ where all edges have cost 1 except $v_iv_k$. Let the terminal set $T=\{v_1,...,v_k\}$ and the predicted terminal set $\predT = \{ v_1,v_k,v_{k+1},...,v_{2k-2}\} $. Thus, the prediction error $\eta=k-2$. 
For the terminal set $T$, we can easily find its optimal solution (illustrated by all solid edges in Fig.~\ref{fig:oapt_lower_bound_restate}). That is the edges $v_1v_i$ for $i \in \{2,3,\ldots, k\}$.  The cost of this solution is $\opt = 1+\frac{k-1}{(k-2)^2}$.

\begin{figure}[hbpt]
    \centering
    \includegraphics[width=0.7\linewidth]{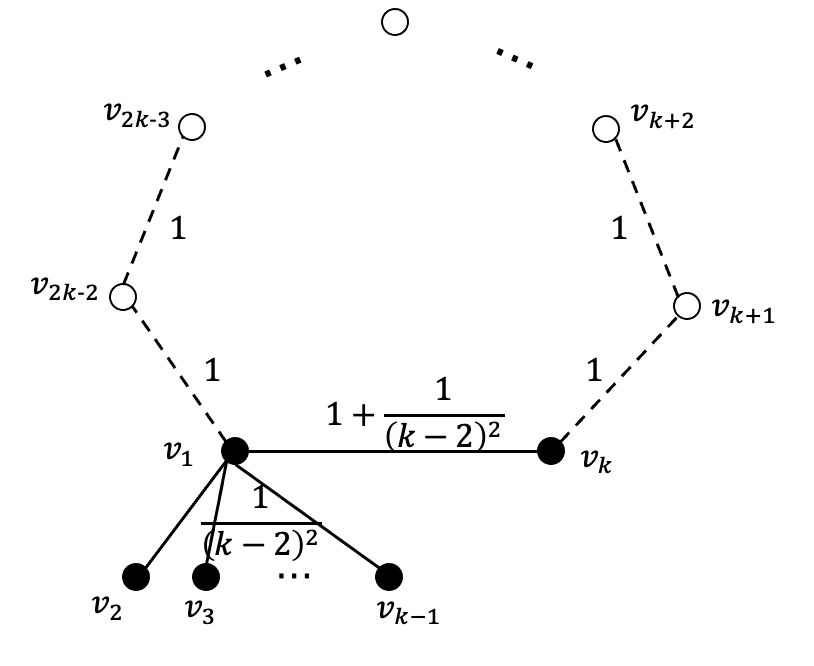}
    \caption{The same online Steiner tree instance as Fig.~\ref{fig:oapt_lower_bound}. The graph has $2k-2$ nodes and $2k-2$ edges. 
    The cost of edge $(v_1,v_k)$ is $1+1/(k-2)^2$ while each edge between $v_1$ and $v_i$ ($2\leq i\leq k-1$) has a cost $1/(k-2)^2$. The cost of each dash edge is 1. The terminal set $T$ consists of $v_1,v_2,...,v_k$ while the prediction is $\predT = \{v_1,v_k,v_{k+1},...,v_{2k-2}\}$.}  
    \label{fig:oapt_lower_bound_restate}
\end{figure}

Consider OAPT. The algorithm first computes its minimum spanning tree, which is represented by $k-1$ dash edges in Fig.~\ref{fig:oapt_lower_bound_restate}. This is the edges of the cycle $v_1, v_k, v_{k+1}, \ldots, v_{2k-2}, v_1$. The cost of $\mst(\predT)$ is $k-1$. 

Let the terminals arrive in the order of $v_1,v_k,v_2,v_3,...,v_{k-1}$. When the second terminal $v_k$ arrives, algorithm OAPT finds $v_{k}$ is a terminal in $\predT$ and add the shortest path from $v_k$ to $v_1$ in $\mst(\predT)$ , which is the path $(v_k,v_{k+1},...,v_{2k-2},v_1)$. The cost gained by this step is the cost of $\mst(\predT)$, i.e., $k-1$. For each remaining terminal, algorithm OAPT adds the shortest edge with cost $1/(k-2)^2$. Thus, the final cost of algorithm OAPT is $k-1+\frac{1}{k-2}=k-2+\frac{k-1}{k-2} = (k-2)\opt = \eta \opt.$
\end{proof}

\begin{proof}[\textbf{Proof of Theorem~\ref{thm:directed_oapt}}]
The analysis of $c(A_1)$ is  easy. We know  the shortest path from $t_i$ to $r$ is always no larger than $\opt$ and the first case occurs $\eta$ times, thus   $c(A_1) \leq \eta \opt.$

Now consider edges in $A_2$. According to the definition of $\predT(\lambda)$, the total cost of $\mdst(\predT(\lambda))$ is at most $k\lambda$.  This is because a feasible solution can be obtained easily by choosing the shortest path for each terminal, whose objective value is at most $k\lambda$. 

When $\lambda > \opt / k^2$, the value of $\lambda$ doubles at most $2\log k$ times since $\lambda \leq 2\opt$ according to the doubling condition. Each time $\lambda$ doubles, a new tree $\mdst(\predT(\lambda))$ is generated. Ideally, if we can bound each $\mdst(\predT(\lambda))$ by a constant factor of $\opt$, $c(A_2)$ can be bounded by $O(\log k) \opt$.



Let $\ell$ be the smallest integer such that $2^\ell > \frac{\opt}{k^2}$. 
The following summations are over the different powers of two that $\lambda$ can take on. 

\begin{equation*}
    \begin{aligned}
        c(A_2) &= \sum_{i=1}^{\ell-1} c(\mdst(\predT(2^i))) + \sum_{i = \ell}^{\ell + 2\log k}c(\mdst(\predT(\lambda))) \\
        & \leq \sum_{i=1}^{\ell-1} k 2^i + \sum_{i = \ell}^{\ell + 2\log k} (\opt + 2^{i}\eta) \;\;\;\; \mbox{[Lemma~\ref{lem:directed_oapt}]} \\
     & = k 2^\ell + \sum_{i = \ell}^{\ell + 2\log k} (\opt + 2^{i}\eta) \;\;\;\;  \\
     & \leq \frac{2\opt}{k} + \sum_{i = \ell}^{\ell + 2\log k} (\opt + 2^{i}\eta) \;\;\;\; \mbox{[Definition of $\ell$]} \\
 & = \frac{2\opt}{k} + (2\log k)\opt + \sum_{i = \ell}^{\ell + 2\log k}  2^{i}\eta\\
  & \leq \frac{2\opt}{k} + (2\log k)\opt + (2\log k)\eta\\
  &\;\;\;\; \mbox{[ $2^{\ell +2\log k} \leq 2 \opt$]}
    \end{aligned}
\end{equation*}


\end{proof}

Now we state a theorem to show that our dependence of $\eta$ in the directed case is optimal.
\begin{theorem}\label{thm:directed_lower_bound}
For online directed Steiner tree with predicted terminals, given any $\eta\geq 1$, no online algorithm has a competitive ratio better than $\Omega(\eta)$.
\end{theorem}

\begin{proof}[\textbf{Proof of Theorem~\ref{thm:undirected_lower_bound} and Theorem~\ref{thm:directed_lower_bound}}]
The two theorems share the same basic idea. 
Thus, we only take the undirected case for an example here.

Fix any $\eta$. Let $I$ be a known hard online instance on $\eta$ terminals used to give the $\Omega(\log(\eta))$ lower bound. Consider combining $I$ with $k-\eta$ terminals with $0$ cost edges. There are totally $k$ terminals. 
The $k-\eta$ zero-cost terminals arrive first and then the $\eta$ terminals arrive in adversarial order according to the hard instance $I$.

No online algorithm can have a competitive ratio better than $\Omega(\log (\eta))$ for this instance. Otherwise, such an algorithm would contradict the $\Omega(\log(\eta))$ lower bound for undirected Steiner tree. 
\end{proof}

\paragraph{Removing the assumption that $|\predT|=|T|$.} Notice that if $|\predT|$ is far larger than $|T|$, the error $\eta := k-|\predT \cap T|$ may not measure the quality of the prediction accurately. 
Thus, we redefine the prediction error $\eta':= \max (|\predT|,|T|)-|\predT \cap T|$ when removing the assumption that $|\predT|=|T|$. The analysis of our algorithms still holds given this new error. The number of terminals not in the prediction and the number of edges bought wrongly due to the prediction (e.g. the set $B_2$ in the undirected case) are still at most $\eta$.  This implies that the same competitive ratios can be obtained with  effectively the same analysis. For the robustness of our algorithm, our proofs have shown that algorithm IOAPT and the directed case algorithm will never perform asymptotically worse than the best traditional algorithms no matter what the prediction is given. Thus, regardless of the definition of the error, their robustness ratios are always $O(\log |T|)$ and $O(|T|)$, respectively.
\section{Additional Experimental Results}\label{sec:more_exp}

\subsection{Robustness to Accuracy in Different Settings}\label{sec:more_terminals}

This subsection shows the robustness performance of algorithms on random graphs under different $k$'s, , where $k$ is the number of terminals, and the performance on different road graphs fixing $k=2000$. From Fig.~\ref{fig:random_graph_more_k} and Fig.~\ref{fig:robust_road_graphs}, they share roughly the same trend. 

\begin{figure}[htbp]
\centering
\subfigure[$k=400$]{
\begin{minipage}{.3\textwidth}
    \centering
    \includegraphics[width=\textwidth]{ 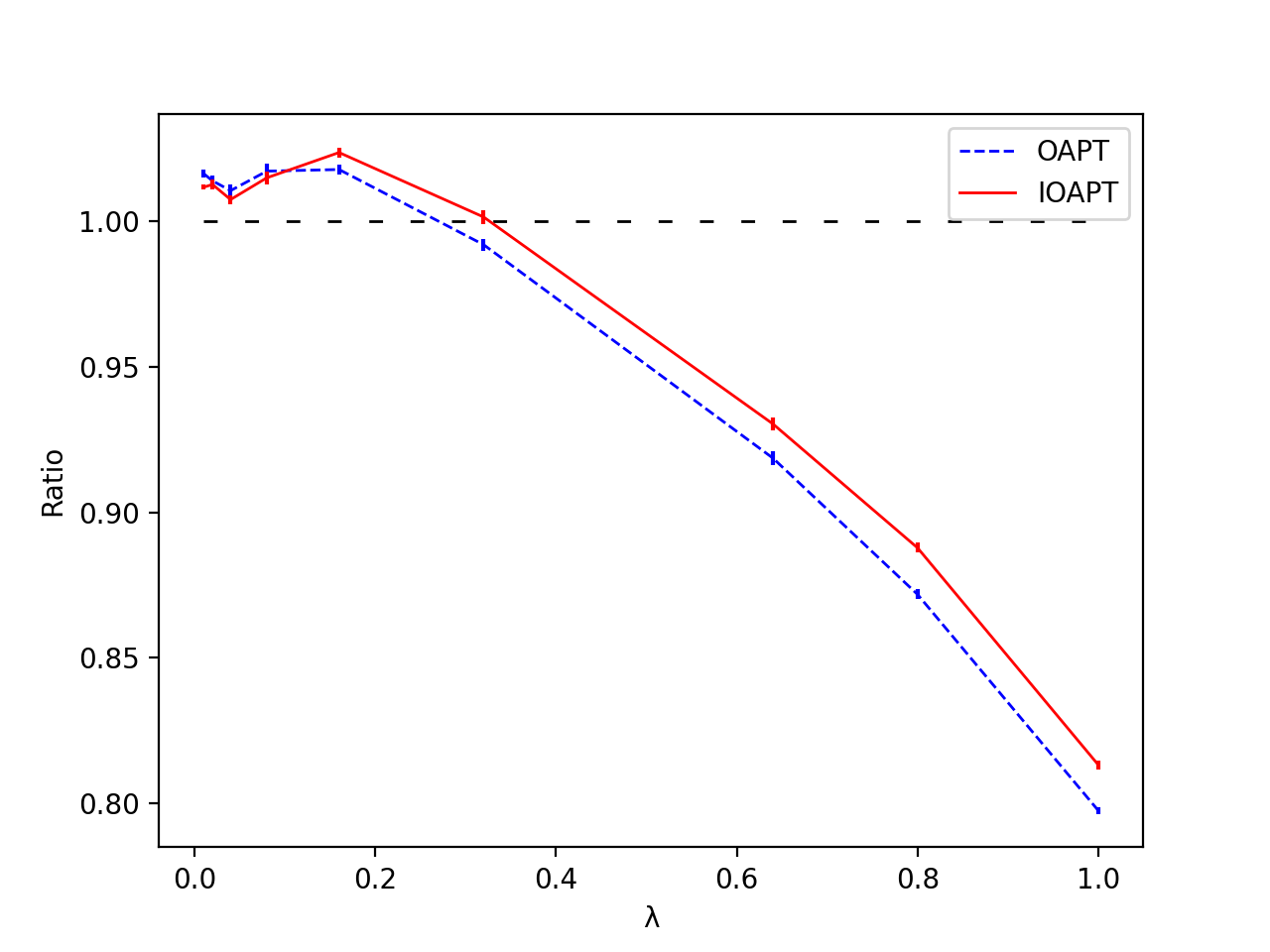}
    \label{fig:k400}
    \end{minipage} }
\subfigure[$k=600$]{
\begin{minipage}{.3\textwidth}
    \centering
    \includegraphics[width=\textwidth]{ 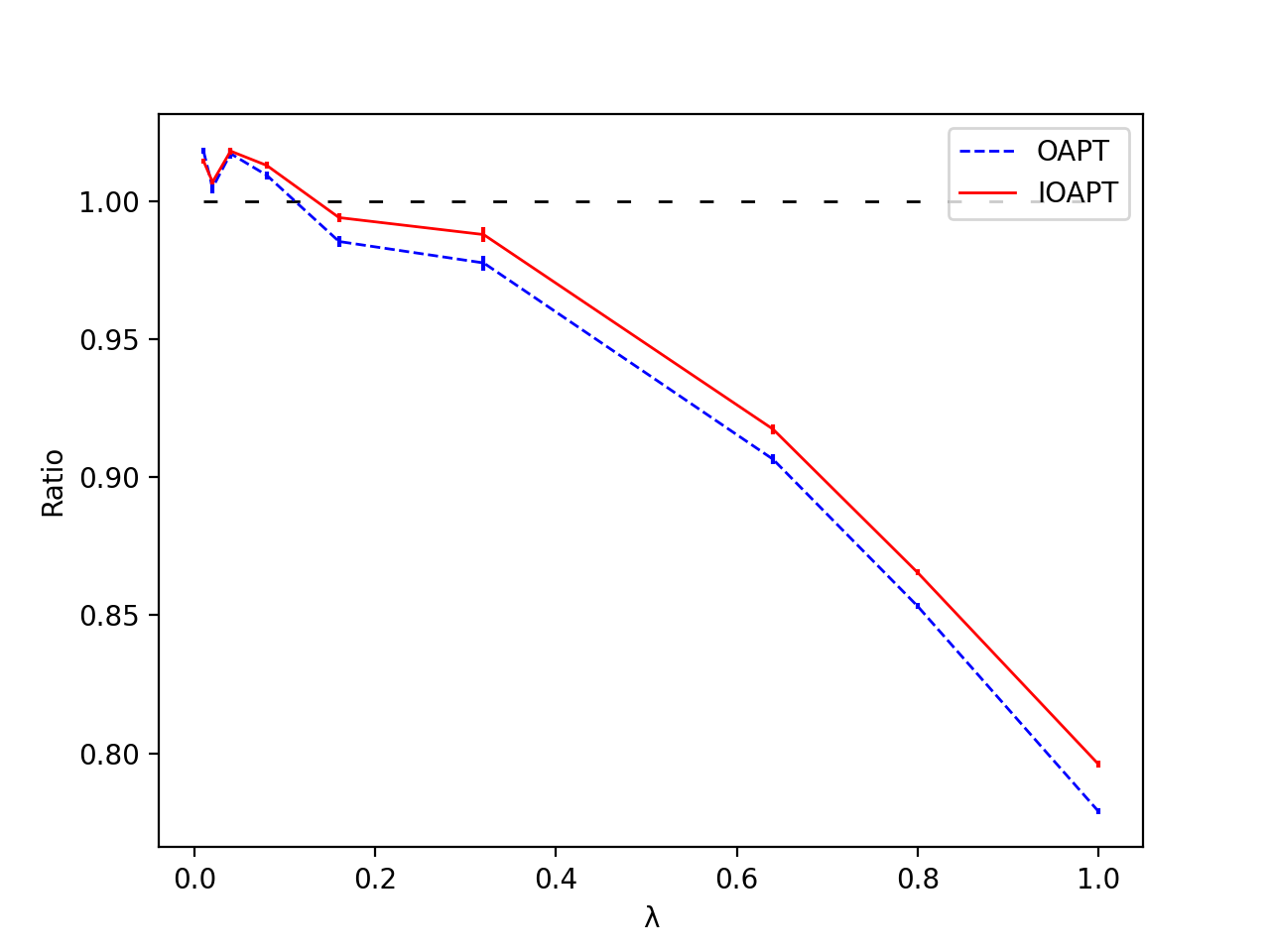}
    \label{fig:k600}
    \end{minipage}}
    
\subfigure[$k=800$]{  
    \begin{minipage}{.3\textwidth}
    \centering
    \includegraphics[width=\textwidth]{ 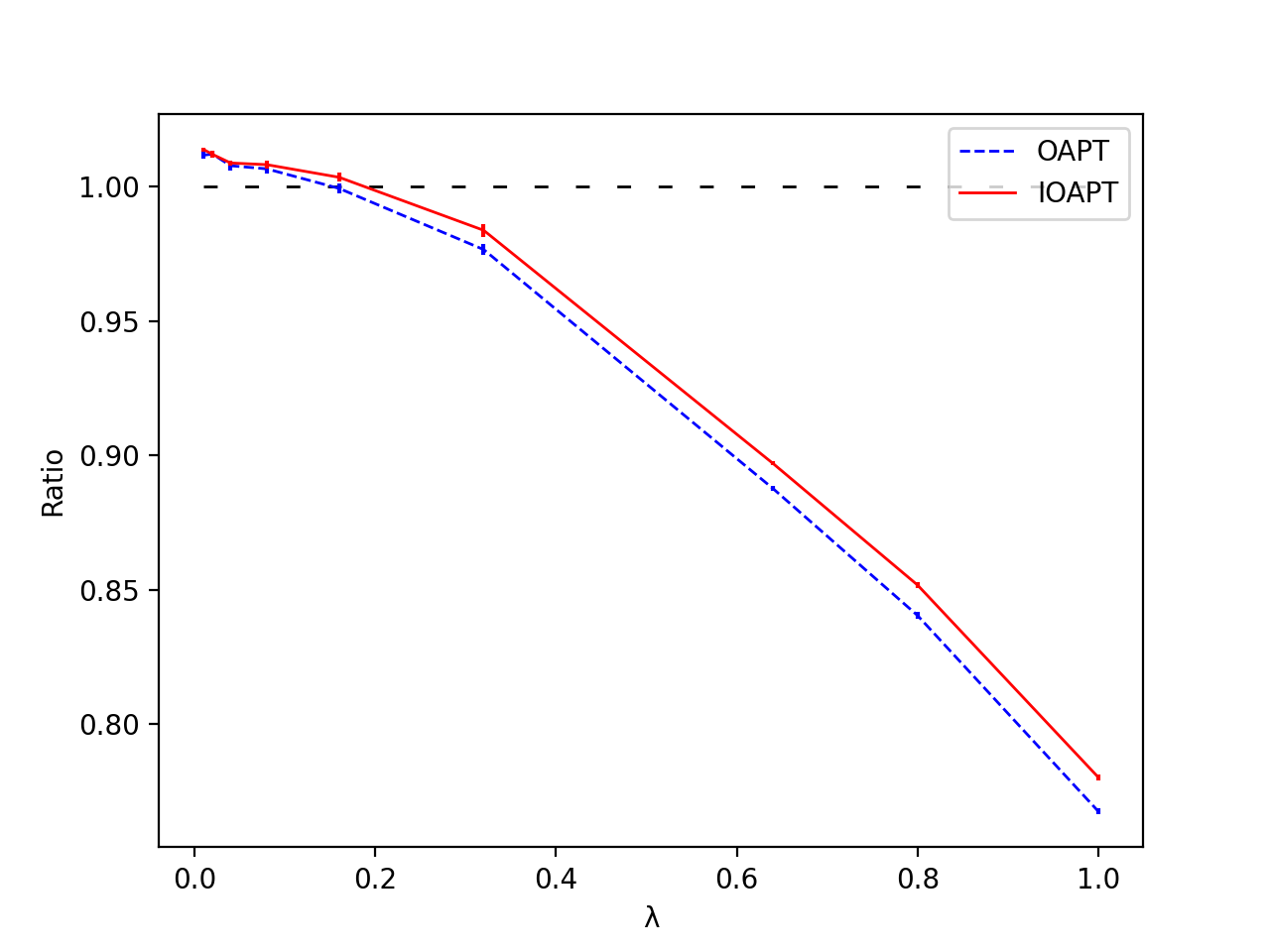}
    \label{fig:k800}
    \end{minipage}}
\subfigure[$k=1000$]{  
    \begin{minipage}{.3\textwidth}
    \centering
    \includegraphics[width=\textwidth]{ 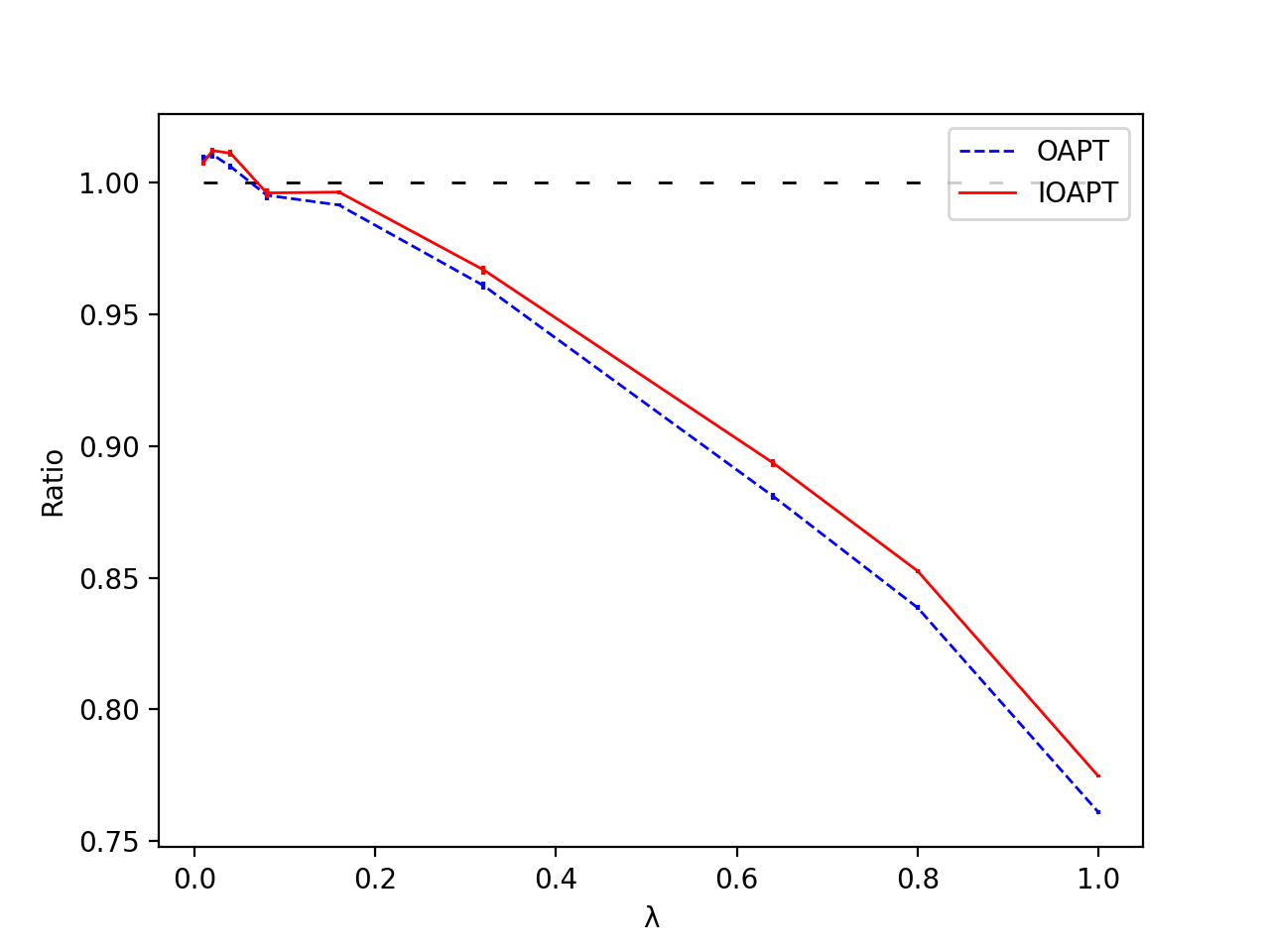}
    \label{fig:k1000}
    \end{minipage}}
    \caption{The performance of algorithms over different $\lambda$'s with different values of $k$.
    }
    \label{fig:random_graph_more_k}
\end{figure}

\begin{figure}[htbp]
\centering
\subfigure[]{
\begin{minipage}{.3\textwidth}
    \centering
    \includegraphics[width=\textwidth]{ Road_graph_robust_exp.png}
    \label{fig:robust_graph0}
    \end{minipage}}
\subfigure[]{
\begin{minipage}{.3\textwidth}
    \centering
    \includegraphics[width=\textwidth]{ 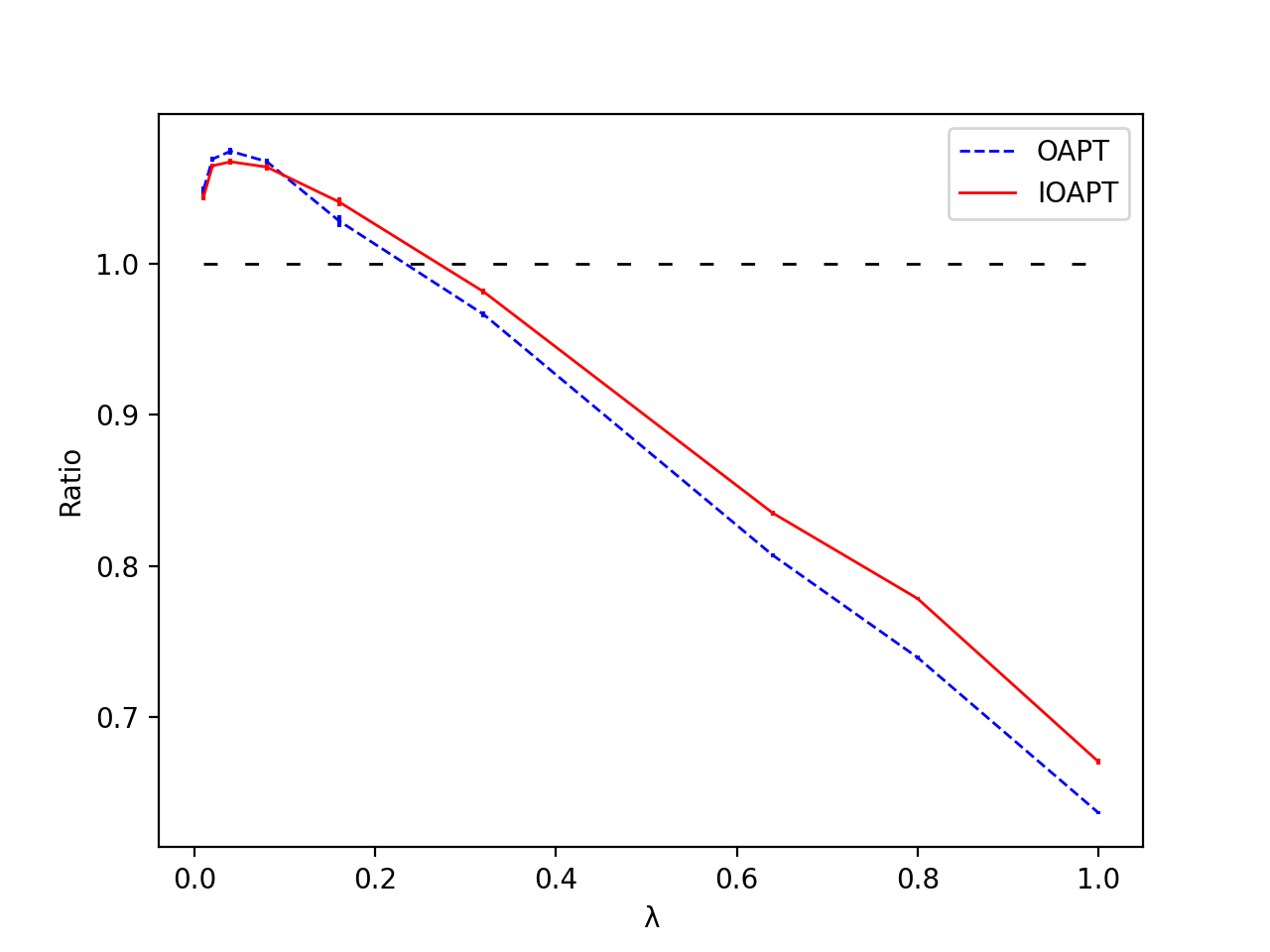}
    \label{fig:robust_graph1}
    \end{minipage}}
    
\subfigure[]{  
    \begin{minipage}{.3\textwidth}
    \centering
    \includegraphics[width=\textwidth]{ 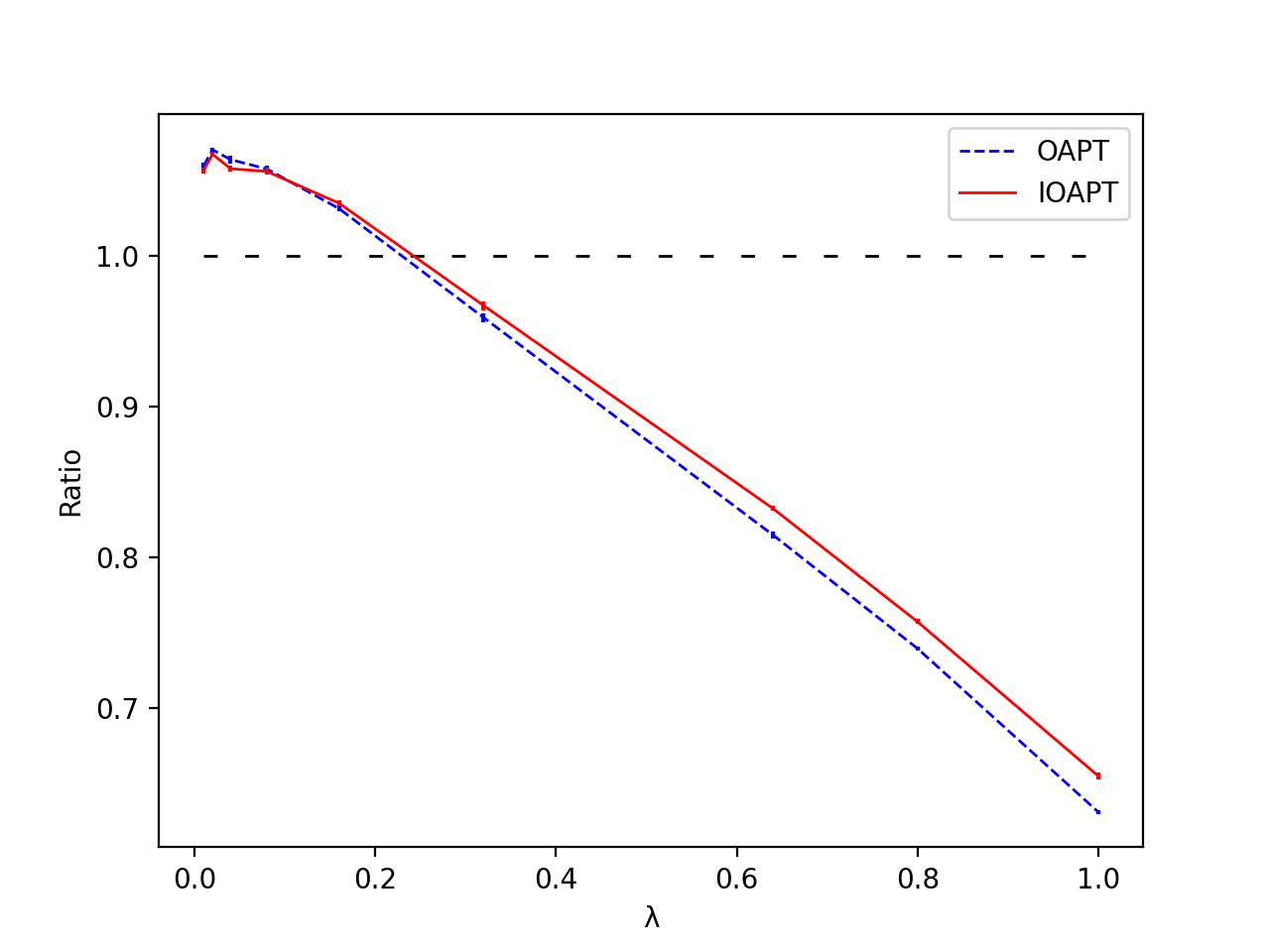}
    \label{fig:robust_graph2}
    \end{minipage}}
\subfigure[]{  
    \begin{minipage}{.3\textwidth}
    \centering
    \includegraphics[width=\textwidth]{ 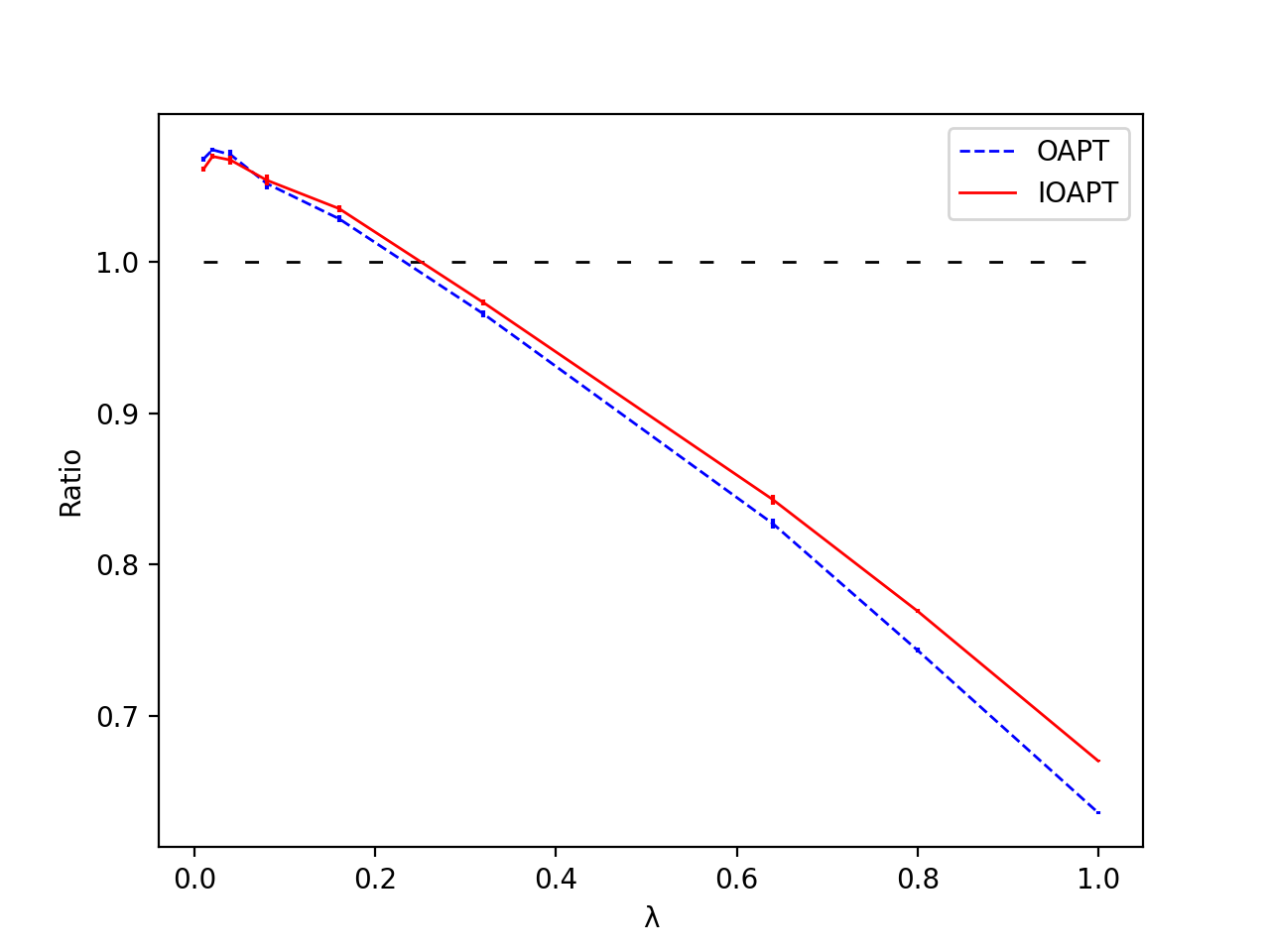}
    \label{fig:robust_graph3}
    \end{minipage}}
    \caption{The performance of algorithms over different $\lambda$'s on different road graphs. Note that Fig.~\ref{fig:robust_graph0} is the road graph shown in the main body.
    }
    \label{fig:robust_road_graphs}
\end{figure}

\subsection{Prediction Errors in the Learning Experiments}\label{sec:prediction_error}

\begin{figure}[htbp]
\centering
\subfigure[Corresponding to Fig.~\ref{fig:random_uniform}]{
\begin{minipage}{.3\textwidth}
    \centering
    \includegraphics[width=\textwidth]{ 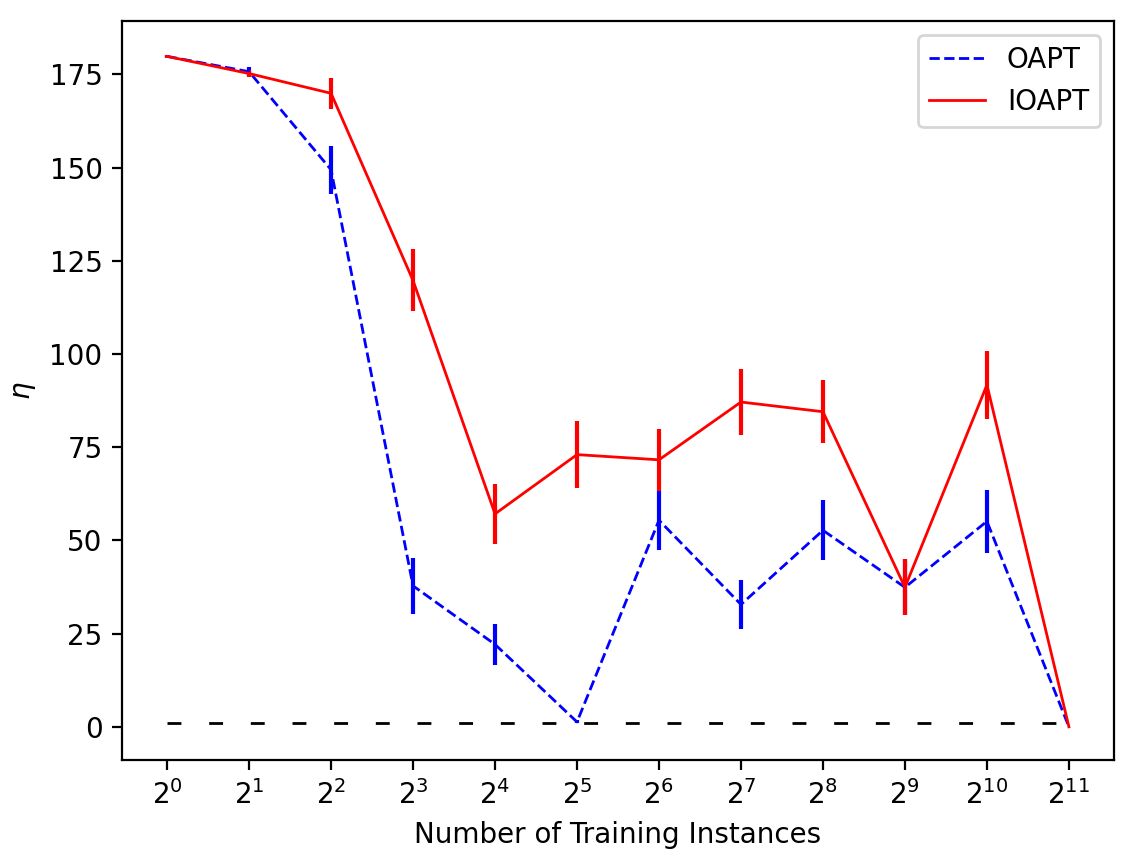}
    \label{fig:random_uniform_error}
    \end{minipage} }
\subfigure[Corresponding to Fig.~\ref{fig:random_good_distri}]{
\begin{minipage}{.3\textwidth}
    \centering
    \includegraphics[width=\textwidth]{ 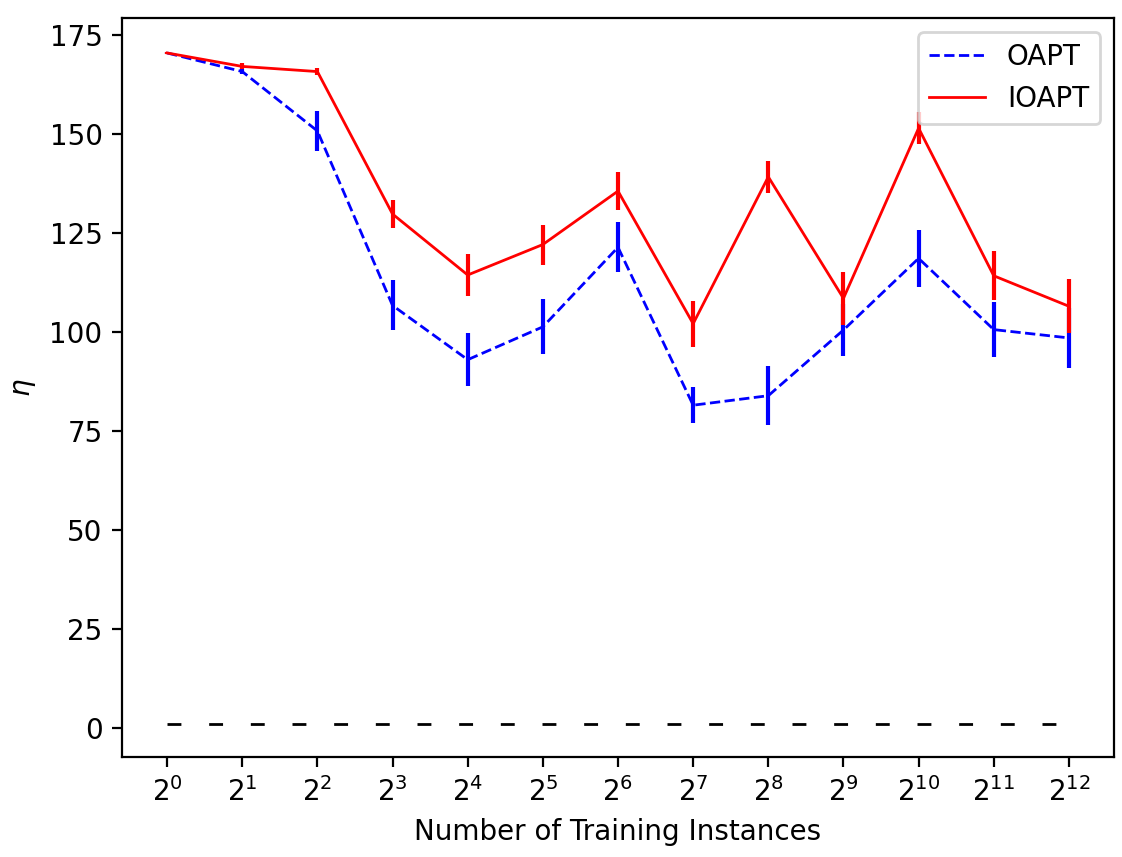}
    \label{fig:random_good_distri_error}
    \end{minipage}}
    
\subfigure[Corresponding to Fig.~\ref{fig:road_cluster_x10}]{  
    \begin{minipage}{.3\textwidth}
    \centering
    \includegraphics[width=\textwidth]{ 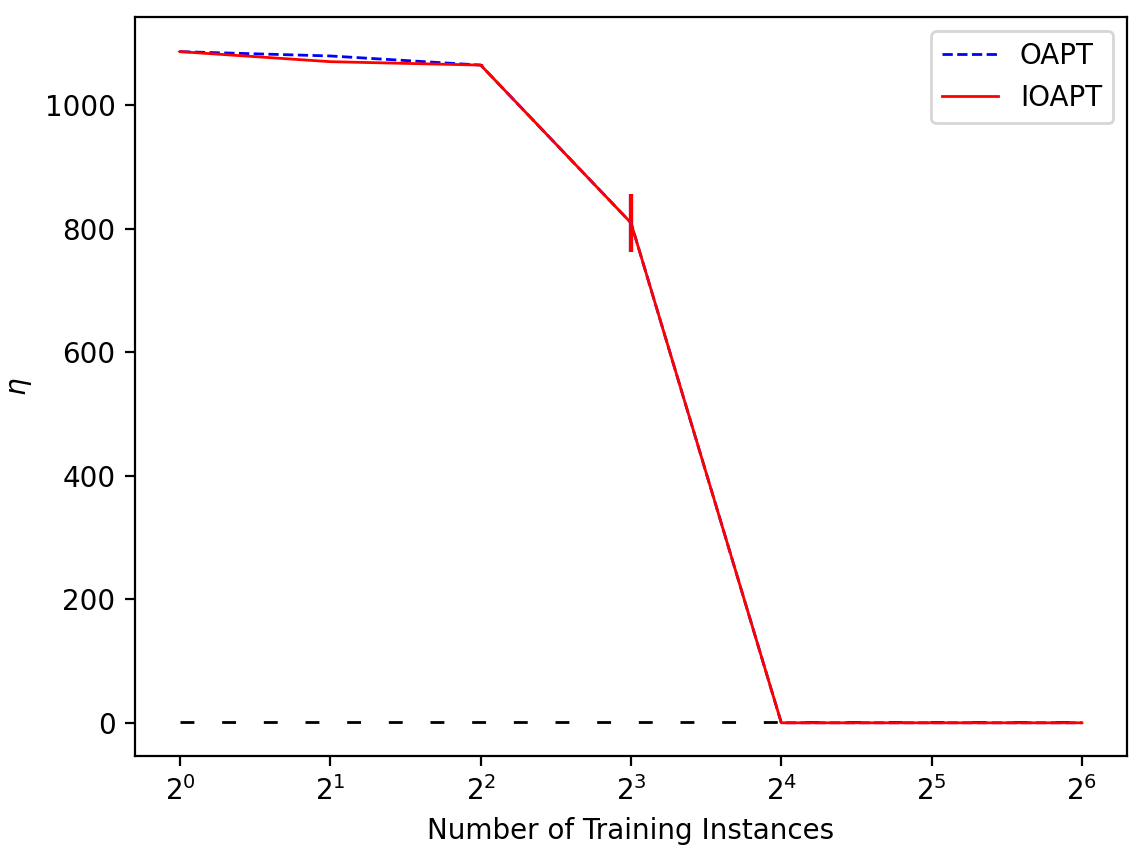}
    \label{fig:road_x10_error}
    \end{minipage}}
\subfigure[Corresponding to Fig.~\ref{fig:road_cluster_x10}]{  
    \begin{minipage}{.3\textwidth}
    \centering
    \includegraphics[width=\textwidth]{ 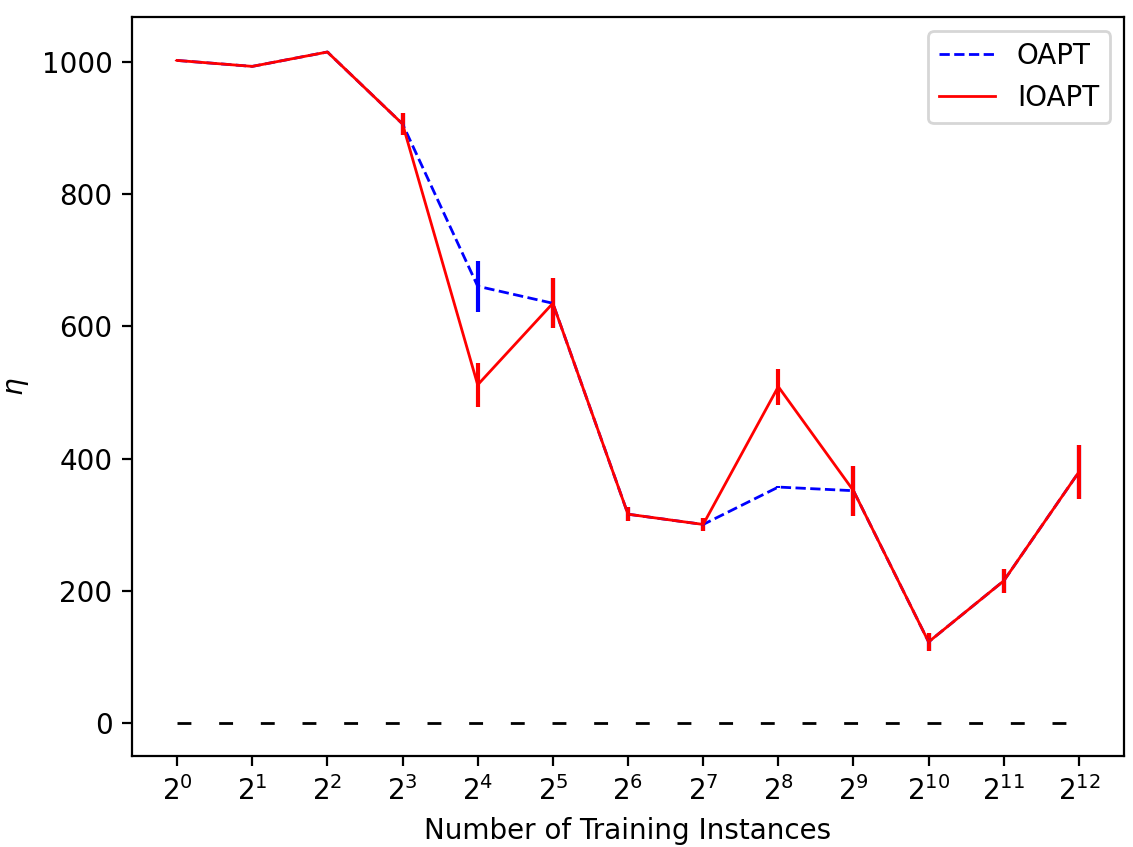}
    \label{fig:road_x100_error}
    \end{minipage}}
    \caption{The prediction error $\eta$ during the training.
    }
    \label{fig:pred_error}
\end{figure}

We present the prediction errors over different numbers of training instances. Noting that the number of predicted terminals may not equal $k$, we define $\eta$ be the number of wrong predicted terminals here, i.e., $\eta = |\predT \setminus T|$. We see the following trends:
\begin{itemize}
    \item For bad distributions (e.g. Fig.~\ref{fig:random_uniform_error}~\ref{fig:road_x10_error}), the learning algorithm predicts almost no terminals after observing tens of instance, resulting in tiny $\eta$ in the figures. Recalling the performance shown in Section~\ref{sec:exp}, when $|\predT|$ is small, the augmenting algorithms will switch to the greedy algorithm automatically and make the performance robust.
    
    \item For good distributions (e.g. Fig.~\ref{fig:random_good_distri_error}~\ref{fig:road_x100_error}), the learning algorithm quickly learns the nodes which become terminals with high probability. Although these predictions are still modestly accurate, they are sufficient for our algorithms to obtain comparable performance (recall Fig.~\ref{fig:random_good_distri}~\ref{fig:road_cluster_x100} in the main body).
\end{itemize}

\end{document}